\documentclass[UKenglish,cleveref,numberwithinsect,a4paper,11pt]{article}

\RequirePackage{doi}
\usepackage[margin=1in]{geometry}
\usepackage{hyperref}
\usepackage{bm}
\usepackage{afterpage}
\usepackage{diagbox}
\usepackage[fleqn]{amsmath}
\usepackage{amssymb,amsfonts,amsthm}
\usepackage{mathtools}
\usepackage[mathlines]{lineno}
\usepackage{thmtools,thm-restate}
\usepackage{tabularx}
\usepackage{url}
\usepackage{xspace}
\usepackage{graphicx}
\usepackage[utf8]{inputenc}
\usepackage{numprint}
\usepackage{microtype}
\usepackage{dsfont}
\usepackage{color}
\usepackage{bbm}
\usepackage{paralist}
\usepackage{grffile}
\usepackage{caption}
\usepackage{subcaption}
\usepackage{nicefrac}
\usepackage{booktabs}

\DeclareMathOperator*{\argmax}{arg\,max}

\DeclareMathOperator*{\cost}{\operatorname{cost}}
\DeclareMathOperator{\depth}{depth}

\usepackage{algorithm}
\usepackage{xcolor}
\usepackage[noend]{algpseudocode}
\usepackage{soul}

\title{Fully-Dynamic Decision Trees}

\begin{document}

\newtheorem{theorem}{Theorem}[section]
\newtheorem{corollary}{Corollary}[theorem]
\newtheorem{lemma}[theorem]{Lemma}
\newtheorem{claim}[theorem]{Claim}
\newtheorem{definition}[theorem]{Definition}
\newcommand{\del}{\textsc{del}}
\newcommand{\ins}{\textsc{ins}}
\newcommand{\lab}{\textsc{lab}}
\newcommand{\scA}{\mathcal{A}}
\newcommand{\scX}{\mathcal{X}}
\newcommand{\scE}{\mathcal{E}}
\newcommand{\scY}{\mathcal{Y}}
\newcommand{\scD}{\mathcal{D}}
\newcommand{\scO}{\mathcal{O}}
\newcommand{\scL}{\mathcal{L}}
\newcommand{\scN}{\mathcal{N}}
\newcommand{\scP}{\mathcal{P}}
\newcommand{\scT}{\mathcal{T}}
\newcommand{\R}{\mathbb{R}}
\newcommand{\N}{\mathbb{N}}
\newcommand{\E}{\mathbb{E}}
\newcommand{\bx}{\pmb{x}}
\newcommand{\by}{\pmb{y}}
\newcommand{\bz}{\pmb{z}}
\newcommand{\be}{\pmb{e}}
\newcommand{\beps}{\pmb{\epsilon}}
\newcommand{\eps}{\epsilon}
\newcommand{\Splits}{\Sigma}
\newcommand{\algo}{\textsc{FuDyADT}}
\newcommand{\Ind}{\mathds{1}}
\newcommand{\ED}{\triangle}
\newcommand{\EDR}{\ED^{\!*}}
\newcommand*\xor{\oplus}
\newcommand{\rt}{\operatorname{r}}
\newcommand{\norm}[2]{\| {#2} \|_\textsc{#1}}
\newcommand{\tvd}[1]{\norm{tvd}{#1}}
\newcommand{\evsplit}{\scE_{\operatorname{split}}}
\newcommand{\psplit}{p_{\operatorname{split}}}
\newcommand{\evstop}{{\operatorname{stop}}}
\newcommand{\pstop}{p_{\operatorname{stop}}}
\newcommand{\splitvar}{{\operatorname{split}}}
\newcommand{\splitdist}{P_{\operatorname{split}}}
\newcommand{\Tup}{T_{\uparrow}}
\newcommand{\mauro}[1]{{\color{blue} #1}}
\newcommand{\marco}[1]{{\color{red} #1}}
\newcommand{\ind}{\textsc{Index}}
\newcommand{\AlgoBuild}{\textsc{Build}}
\newcommand{\AlgoUpdate}{\textsc{Update}}
\newcommand{\epsThr}{\alpha}
\newcommand{\epsApx}{\beta}
\newcommand{\poly}{\operatorname{poly}}

\author{
   Marco Bressan\\ Department of Computer Science,\\ University of Milan
   \and
   Gabriel Damay\\ Institut Polytechnique de Paris,\\ Télécom Paris
   \and
   Mauro Sozio\\ Institut Polytechnique de Paris,\\ Télécom Paris
}

\maketitle

\begin{abstract}
We develop the first fully dynamic algorithm that maintains a decision tree over an arbitrary sequence of insertions and deletions of labeled examples. Given $\epsilon > 0$ our algorithm guarantees that, at every point in time, every node of the decision tree uses a split with Gini gain within an additive $\epsilon$ of the optimum. For real-valued features the algorithm has an amortized running time per insertion/deletion of $\scO\big(\frac{d \log^3 n}{\epsilon^2}\big)$, which improves to $\scO\big(\frac{d \log^2 n}{\epsilon}\big)$ for binary or categorical features, while it uses space $\scO(n d)$, where $n$ is the maximum number of examples at any point in time and $d$ is the number of features. Our algorithm is nearly optimal, as we show that any algorithm with similar guarantees uses amortized running time $\Omega(d)$ and space $\widetilde {\Omega} (n d)$. We complement our theoretical results with an extensive experimental evaluation on real-world data, showing the effectiveness of our algorithm. 
\end{abstract}

\section{Introduction}\label{sec:intro}
Decision trees are a cornerstone of machine learning, and an essential tool in any machine learning library. 
Given a feature domain $\scX$ and a label domain $\scY$, a decision tree is a function $f : \scX \mapsto \scY$ that assigns to each $x \in \scX$ a label $y \in \scY$ by traversing a tree $T$ from its root node to a leaf. At each node of the tree, an example $x$ is evaluated by some rule that determines which successor should receive $x$ --- for instance, a common rule is a simple threshold on some feature. Every leaf is associated with a label, which is the result of the prediction when such a leaf is reached. 
The problem of constructing an optimal decision tree is NP-hard w.r.t.\ to several natural objective functions~\cite{shalevshwartz2014understanding}. This has led to the introduction of several heuristic approaches, such as ID3, C4.5, C5.0 and CART, which have proven very effective and are now considered state of the art.
Typically, those approaches proceed in a greedy fashion by selecting for each node a feature and a splitting value (we shall call such a pair a \emph{split}) that optimize some measure of improvement such as the Gini gain or the information gain. This is repeated until a stopping condition is met, such as the tree reaching a certain height or the number of examples at every leaf falling below some threshold.

Recently, there have been significant efforts to adapt machine learning algorithms to a \emph{fully dynamic} setting, where the algorithm is asked to process an arbitrary list of insertions or deletions. Insertions are typically the result of new data being collected or revealed, while deletions can be the result of noise removal, removal of personal data for privacy concerns, data becoming obsolete, etc. It might not be desirable to make any assumption on such a list of update operations, which motivates the design of fully dynamic algorithms. 
Most works in fully-dynamic machine learning algorithms have focused on unsupervised tasks such as clustering~\cite{DBLP:conf/nips/Cohen-AddadHPSS19, DBLP:conf/esa/HenzingerK20, DBLP:journals/corr/abs-2112-07050, DBLP:conf/www/ChanGS18} or graph mining~\cite{DBLP:conf/stoc/SawlaniW20, DBLP:conf/stoc/BhattacharyaHNT15, DBLP:conf/www/EpastoLS15, DBLP:journals/tkdd/StefaniERU17}.

For decision trees, however, only  \emph{incremental} algorithms are known, which handle insertions but not deletions.\footnote{These algorithms are also called \emph{online} decision tree algorithms} The state of the art in this case is given by Hoeffding Trees~\cite{Domingos00-HighSpeedStreams} and their evolutions such as EFDT or HAT (see~\cite{Manapragada2020} for a survey). Not only are these algorithms incapable of handling deletions, but no good bound on their amortized cost is known, while guarantees hold only if the insertions are i.i.d.\ from some distribution. Our work represents one of the first studies on fully-dynamic supervised machine learning which has been mostly unexplored so far, to the best of our knowledge.

Defining what kind of decision tree a dynamic algorithm should maintain requires some care. The first natural attempt is to maintain the very same decision tree that the greedy approaches above (ID3, C4.5, etc.) would produce from the current set of examples. Thus, at any time, every node should use a split with maximal gain with respect to the set of examples held by its subtree. The problem with this goal is that, at some point, the gains of the best and second-best splits may differ by just $\scO(1/n)$ where $n$ is the total number of examples. In that case, $\scO(1)$ updates can turn the second-best split into the best one, possibly forcing a reconstruction of the whole tree. The same happens if one wants splits within multiplicative factors of the best one, as the latter may be in $\scO(1/n)$. Hence, in those cases it is unclear whether there is an efficient fully-dynamic algorithm.
The next natural goal is maintaining a tree with $\epsilon$-optimal splits, that is, within an \emph{additive} $\epsilon$ of the best ones.  We shall call such a decision tree $\epsilon$-feasible. Aiming at an $\epsilon$-feasible tree is reasonable, since excessively small gains are statistically not significant (a gain of, say, $10^{-4}$ is likely the result of noise) and thus approximating large gains is enough. 
Indeed, algorithms such as EFDT or HAT try to maintain precisely an $\epsilon$-feasible tree. However, their approach is based on computing exactly the Gini gains. For real-valued features, doing that for every possible split would lead to an $\scO(n)$ amortized cost, hence they resort to a heuristic which comes at the price of worse results. 

In this work we develop an efficient fully-dynamic algorithm for maintaining an $\epsilon$-feasible decision tree.
Our first observation is that, in order to change by $\epsilon$ the gain of a given split on a sequence of examples $S$, one must make $\Omega(\epsilon |S|)$ insertions or deletions. Thus, one could rebuild a subtree after $\Theta(\eps |S|)$ updates, without even tracking the gains, with the number of updates covering the rebuilding cost. Intuitively speaking, this is one of the arguments we use in our amortized cost analysis, which is pretty standard. However, this yields amortized time bounds that are \emph{quadratic} in the height $h$ of the tree, because a sequence of updates can force a ``cascade'' of rebuilds on $\Theta(h)$ subtrees each having height $\Theta(h)$.
We show how to bypass this obstacle and save a factor of $h$ in the amortized cost with a ``proactive'' strategy that rebuilds subtrees slightly larger than necessary. Through a careful amortized analysis based on a few charging arguments, this yields our fully-dynamic algorithm for real-valued features. For categorical features, our algorithm can be improved via a faster tree reconstruction subroutine. Moreover, our algorithm can satisfy constraints more general than just $\epsilon$-feasibility, including pruning at a certain height or guaranteeing splits only if enough examples are available. Finally, we prove that our algorithms are nearly optimal: no algorithm can beat their time or space usage by more than $\poly\log(nd)$ factors, even if one looks at algorithms attaining considerably weaker guarantees.
Our contributions can be summarized as follows:
\begin{itemize}
\item We present \algo, a deterministic algorithm for maintaining an $\epsilon$-feasible decision tree under an arbitrary sequence of insertions and deletions. It uses $O(nd)$ space, while it has $O\big(\frac{d \log^3 n} {\epsilon^2} \big)$ amortized running time for real-valued features and $O\big(\frac{d \log^2 n} {\epsilon} \big)$ for categorical ones.
\item We prove a lower bound of $\tilde{\Omega} (n d )$ on the space requirements and of $\Omega(d)$ on the amortized running time of any fully dynamic algorithm, even in easier settings. This makes \algo\ optimal up to $\poly \log (nd)$ factors.
\item We conduct an extensive experimental evaluation on real-world data, evaluating \algo's speed and accuracy against state-of-the-art tools such as EFDT and HAT.
\end{itemize}

\noindent\textbf{Related Work.} 
The works closest to ours are those in the incremental setting. Here, the algorithm receives a stream of examples from a distribution, and has to perform well when compared to the offline tree built on the entire sequence. In this setting, Hoeffding trees~\cite{Domingos00-HighSpeedStreams} emerged as one of the most effective approaches, inspiring several variants, even ones capable of handling concept drifts~\cite{Domingos01-MiningTimeSeries,Gama03-HighSpeedStreams,Manapragada18,Das19,ijcai2020-177,Haug22,Jin03-Streaming,Rutkowski13-MiningStreams}; see~\cite{Manapragada2020} for a survey. These algorithms crucially rely on the examples being i.i.d., which allows them to compute good splits with high probability via concentration bounds (whence the name). Moreover, those algorithms cannot handle efficiently real-valued features, since on those features they would update $\Theta(n)$ counters at each time, even when only insertions are allowed. Our algorithms instead efficiently handle arbitrary sequences of insertions and deletions of examples with real-valued features. 

We observe that there are general techniques to turn offline data structures into dynamic ones, see~\cite{BENTLEY1980301}. Those techniques, however, work only for problems that have a special decomposability property --- loosely speaking, the answer to a query (e.g., find $\min(X)$ for some set $X$) must be quickly computable from the answers to sub-queries (e.g., $\min(A \cup B) = \min(\min(A),\min(B))$). In our case, a query corresponds to the label predicted by the tree for a given $x$. Unfortunately, our problem is far from decomposable and it does not seem solvable via such techniques.


\section{Preliminaries}
\label{sec:prelim}
All missing proofs can be found in the Appendix. We denote the feature and label domains respectively by $\scX$ and $\scY$; by default $\scX=\R^d$ and $\scY=\{0,1\}$. We denote by $(x,y) \in \scX \times \scY$ a labeled example, by $x_j$ the value of its $j$-th feature, and by $S$ a multiset of labeled examples. We may treat $S$ as a sequence; this will be clear from the context. We assume examples can be stored in $\scO(d)$ bits, while the $x_j$'s can be accessed in time $\scO(1)$.
We let $S[\ldots]$ be the subset of $S$ matching a condition; e.g. $S[x_j \le t] = \{(x,y) \in S : x_j \le t\}$. A \emph{split} is a pair $(j,t) \in [d] \times \R$. We use the bold font for vectors (e.g.  $\bx$). We use Gini gain to measure split quality.
\begin{definition}\label{def:gini}
The Gini index of $S$ is $g(S) = 2\, p_S (1-p_S)$, where $p_S = \frac{1}{|S|} \sum_{(x,y) \in S} y$. The Gini gain of $(j,t)$ on $S$ is:
\begin{align}
G(S,j,t) = g(S) - \left(\frac{|S_-|}{|S|} g(S_-) + \frac{|S_+|}{|S|} g(S_+)\right)
\end{align}
where $S_-=S[x_j \le t]$ and $S_+=S[x_j > t]$. When $|S|=0$ we define $g(S)=0$ and $G(S,j,t)=0$ for all $(j,t) \in [d] \times \R$.
\end{definition}
For all $j \in [d]$ let $G(S,j) = \max_{t\in \R} G(S,j,t)$. Hence, $\argmax_{j}G(S,j)$ is a feature with maximum Gini gain over $S$. Finally, we let $G(S)=\max_{j \in [d]} G(S,j)$.

We rely on the following smoothness properties of the Gini index and the Gini gain. Given two multisets/sequences $S,S'$, their edit distance $\ED(S,S')$ is the minimum number of insertions and deletions to obtain $S'$ from $S$, and their relative edit distance is $\EDR(S,S') = \frac{\ED(S,S')}{\max(|S|,|S'|)}$. 
\begin{lemma}
\label{lem:smoothness}
Let $S,S'$ be multisets of labeled examples.
\begin{enumerate}
    \item $|G(S,j,t) - G(S',j,t)| \le 12 \EDR(S,S'), \forall (j,t) \in [d] \times \R$
    \item $|g(S)-g(S')| \le 2.5 \, \EDR(S,S')$.
\end{enumerate}
\end{lemma}

\noindent\textbf{Decision trees.}
A decision tree is a triple $(T,\Splits,L)$, where $T=(V,A)$ is a directed binary tree rooted at $r(T)$, and $\Splits$ and $L$ are functions that assign splits to internal nodes and labels to the leaves. More formally $\Splits = \{\sigma_v : v \in V(T)\}$ where $\sigma_v = (j_v,t_v) \in [d] \times \R$ for every internal node $v$ of $T$, while $L=\{L_v : v \in V(T)\}$ where $L_v \in \{0,1\}$ for every leaf $v$ of $T$. For any $x \in \scX$ and any internal vertex $v$ of $T$ let $\operatorname{succ}(v,x)$ be the left child of $v$ if $x_j \le t$ and the right child of $v$ otherwise, where $\sigma_v=(j,t)$. For any $v \in V(T)$ let $P_v=(v_0,\ldots,v_{\ell})$ be the unique path from $v_0=r(T)$ to $v_{\ell}=v$. For a multiset $S$, denote by $S_v$ the set of examples $x \in S$ such that $\operatorname{succ}(v_i,x)=v_{i+1}$ for all $i=0,\ldots,\ell-1$; this is the subset of $S$ \emph{associated} to $v$. For every $x \in \scX$ let $v(x)$ be the leaf $x$ is associated to. The labeling given by $T$ is the function $T:\scX \to \scY$ such that $T(x)=L_{v(x)}$ for every $x \in \scX$.
We denote by $T_v$ the subtree of $T$ rooted at $v$ and by $(T,\Splits,L)_{v}$ the decision subtree rooted at $v$.

\noindent\textbf{Algorithms.}
A fully-dynamic decision tree algorithm $\scA$ is defined as follows. The input of $\scA$ is an  \emph{update sequence} $U$ of requests of three types: insertion, $\ins(x,y)$; deletion, $\del(x,y)$; labeling, $\lab(x)$. Each such sequence $U$ induces an \emph{active multiset} of labeled examples $S$ obtained by inserting/deleting the examples following the order of the sequence.
Suppose $\scA$ has processed an update sequence $U$. We say $\scA$ is \emph{coherent} with a decision tree $T$ if, for every $x \in \scX$, any further request $\lab(x)$ makes $\scA$ output $T(x)$. The \emph{query time} of $\scA$ is the worst-case time it takes to $\scA$ to output $T(x)$. Our goal is to construct a fully dynamic algorithm $\scA$ that has low query time and, at every point in time, is coherent with a decision tree $T$ that is $\beps$-feasible with respect to the current active set $S$ (see below).

\section{A Fully Dynamic Decision Tree Algorithm}\label{sec:dyn}


This section presents \algo\ (Fully Dynamic Amortized Decision Tree).  As argued in Section~\ref{sec:intro}, one of our goals is to ensure that every node of the tree uses a split whose gain is within an additive $\eps$ of the maximum. \algo\ satisfies a stricter guarantee, called $\beps$-feasibility, which allows to also prune the tree at some height or at leaves with few examples. 

\begin{definition}\label{def:eps_feasibility}
Let $k,h \in \N$, and let $\beps=(\epsThr,\epsApx)$ where $\epsThr,\epsApx \in (0, 1]$. A decision tree $(T,\Splits,L)$ is $\beps$-feasible, with pruning thresholds $(k,h)$, w.r.t.\ a multiset $S$ of labeled examples if for every $v \in V(T)$:
\begin{enumerate}
\item if $|S_v| \le k$ or $g(S_v) = 0$ or $\depth_T(v)=h$ then $v$ is a leaf, else if $g(S_v) \ge \epsThr$ then $v$ is an internal node
\item if $\sigma_v = (j,a)$ then $G(S_v,j,a) \geq G(S_v,j',a') - \epsApx$ for all $(j',a') \in [d] \times \R$
\item if $v$ is a leaf then $L_v$ is a majority label of $S_v$
\end{enumerate}
\end{definition}
For any fixed pruning thresholds $k,h$ we say that a fully dynamic algorithm $\scA$ is $\beps$-feasible if, at any point in time, $\scA$ is coherent with a decision tree $(T,\Splits,L)$ that is $\beps$-feasible with respect to the current active set. 
When $k=1$ and $h=\infty$ and $\epsThr=\epsApx$, $\beps$-feasibility reduces to the following condition: if $g(S_v)=0$ then $v$ is a leaf, and if $g(S_v) \ge \epsThr$ then $v$ is internal and use an $\epsThr$-optimal split. This is the $\epsilon$-optimality condition of Section~\ref{sec:intro} used by incremental algorithms such as Hoeffding trees and EFDT.
We prove:

\begin{theorem}\label{thm:main_UB}
Let $\scX=\R^d$, let $k,h$ be positive integers, let $\epsThr,\epsApx \in (0,1]$, and let $0 < \epsilon < \min\!\big(\frac{1}{k+1},\frac{\epsThr}{5},\frac{\epsApx}{12.5}\big)$. There is a deterministic $(\epsThr,\epsApx)$-feasible fully dynamic decision tree algorithm with pruning thresholds $k,h$ that has query time $O(h^*)$, uses space $O(nd)$, and has amortized running time per update $O\big(\frac{d h^{\!*} \log^2 n}{ \epsilon}\big) = O\big(\frac{d \log^3 n}{\epsilon^2}\big)$, where $h^*\le h$ and $n$ are respectively the maximum height of the tree and the maximum size of the active set at any time.
\end{theorem}

Theorem~\ref{thm:main_UB} can be improved for categorical features, that is, when $\scX=A^d$ for some fixed finite set $A$; this includes the case of binary features, $A=\{0,1\}$. 
\begin{theorem}\label{thm:categorical_ub}
If $\scX=A^d$ for a finite set $A$ then the amortized time bound of Theorem~\ref{thm:main_UB} can be improved to $O\big(\frac{d \log^2 n}{\epsilon}\big)$.
\end{theorem}
The rest of this section describes \algo\ and proves Theorem~\ref{thm:main_UB}, except for the $\beps$-feasibility part, which is proven in section \ref{sec:addexp}, as is Theorem~\ref{thm:categorical_ub}.
Before moving on, let us give some intuition on \algo.
The algorithm consists of the two routines \AlgoUpdate\ and \AlgoBuild\ below. Those routines maintain a decision tree $T$, and, for each leaf $v$ of $T$, dictionaries $D_v$ and $D_v^L$ storing respectively the multiset $S_v$ associated to $v$ and the frequency histogram of the labels of $S_v$. At every insertion or deletion of an example $(x,v)$, \AlgoUpdate\ computes the leaf $v=v(x)$ where $x$ ends up, and updates $D_v$ and $D^L_v$ consequently, see line~\ref{line:update_dicts}. Then, \AlgoUpdate\ checks if any subtree should be rebuilt. To this end, for every vertex $u \in V(T)$ it maintains two counters, $s(u)$ and $c(u)$, storing respectively the size of the multiset on which the subtree $T_u$ was rebuilt the last time and the number of updates that reached $u$ since that time. As soon as $c(u) > \epsilon \cdot s(u)$ for some $u \in V(T)$, \AlgoUpdate\ invokes \AlgoBuild\ to rebuild an appropriately chosen supertree of $T_u$.

\begin{algorithm}[h]
\caption{\algo.\AlgoUpdate} \label{alg:fullydyndec}
\begin{algorithmic}[1]
\Procedure{Update}{($T,\Splits,L$), ($x,y$), $o$}
\State $P_{v_{\kappa_\ell}} \gets v_{\kappa_1},\dots, v_{\kappa_\ell}$ with  $v_{\kappa_1}=r(T)$, $v_{\kappa_\ell}=v(x)$
\State update $D_{v_{\kappa_\ell}}$ and $D^L_{v_{\kappa_\ell}}$ according to $(x,y),o$ \label{line:update_dicts}
\State $L_{v_{\kappa_\ell}} \gets$ any majority label in $D^L_{v_{\kappa_\ell}}$
\For{$i=1,\dots,\ell$}
\State $c({v_{\kappa_i}}) \gets c(v_{\kappa_i})+1$ \label{line:inc_c}
\If{$c({v_{\kappa_i}}) > \epsilon \cdot s ( v_{\kappa_i}) $} 
\label{line:if_c_eps_s}
\State $\hat{s} \gets 2^{\left\lceil \log s ( v_{\kappa_i})  \right\rceil }$ 
\label{line:hat_s}
\State $j \gets \min \{ j' \in \{0,\ldots,i\} \,:\, s(v_{\kappa_{j'}}) \leq \hat{s} \}$ \label{line:min_j}
\State $(T',\Splits',L') \gets$ \AlgoBuild($S_{v_{\kappa_j}},i$)
\State $(T,\Splits,L)_{v_{\kappa_j}} \gets (T',\Splits',L')$
\State \textbf{return}
\EndIf
\EndFor
\EndProcedure
\end{algorithmic}
\end{algorithm}

\begin{algorithm}[h]
\caption{\algo.\AlgoBuild} \label{alg:build}
\begin{algorithmic}[1]
\Procedure{Build}{$S,\eta$} 
\State $r \gets$ new vertex,~~~$c(r) \gets 0$,~~~$s(r) \gets |S|$ 
\label{line:creation}\label{line:sv}
\If{$|S|\le k$ or $g(S) \le \frac{\epsThr}{2}$ or $\eta = h$} \label{line:if_split}
\State store $S$ in a dynamic dictionary $D_r$
\State and its labels in a dynamic dictionary $D^L_r$
\State $(T,\Splits,L) \gets$ decision tree with $T=(\{r\},\emptyset)$
\State $L_r \gets $ any majority label in $D^L_r$
\ElsIf{$g(S) > \frac{\epsThr}{2}$}
\State $(j,a) \gets \arg \max \{G(S,\hat \iota, \hat a) : (\hat \iota, \hat a) \in [d] \times \R\}$
\State $T_1 \gets $ \AlgoBuild($S[x_j \le a], \eta + 1$)
\State $T_2 \gets $ \AlgoBuild($S[x_j > a], \eta + 1$)
\State $(T,\Splits,L) \gets$ decision tree with root $r$, $T_1,T_2$ as left, right subtrees, and split $\sigma_r=(j,a)$
\EndIf
\State \textbf{return} $(T,\Splits,L)$
\EndProcedure
\end{algorithmic}
\end{algorithm}


Let us move to the bounds of Theorem~\ref{thm:main_UB}. Proving those bounds requires some care in charging the cost of rebuilding the subtrees to the update requests. To this end, we need the following two simple results proven in Appendix~\ref{apx:dyn}. From now on, by ``time $t$'' we mean the $t$-th invocation of \AlgoUpdate.
\begin{lemma}\label{lem:counter}
Let $(T,\Splits,L)$ be the result of $t \ge 0$ invocations of \AlgoUpdate. Then $(1-\epsilon) \cdot s^t(v) \leq |S^t_v| \leq (1+\epsilon) \cdot  s^t(v)$ for every $v \in V(T)$, where $s^t(v)$ is the value of $s(v)$ at time $t$.
\end{lemma}
\begin{lemma}\label{lem:log_height}
Let $(T,\Splits,L)$ be a decision tree built on a sequence $S$. If every $v \in V(T)$ uses a split with gain at least $\gamma > 0$ w.r.t.\ $S_v$, then $T$ has height  $\scO \left( \log |S|  \;/\; {\gamma}  \right)  $. 
\end{lemma}
We can now prove:
\begin{lemma}\label{lem:amtime}
\AlgoBuild\ and \AlgoUpdate\ can be implemented so that any $\mathcal{T}$ invocations of \AlgoUpdate\ take time 
\begin{align}
    O\!\left(\frac{\scT \cdot d \cdot h \cdot (\log n)^2}{\epsilon}\right) = O\!\left(\frac{\scT \cdot d \cdot (\log n)^3}{\epsilon^2}\right)
\end{align}
\end{lemma}
\begin{proof}

First we describe the data structures and the time taken by the basic operations of \AlgoBuild\ and \AlgoUpdate. Using a self-balancing tree for $D_v$ we ensure search, insert, update, and deletion in time $O(d \log N)$, and enumeration in time $O(d N)$ --- recall that every element takes $O(d)$ bits --- where $N$ is the number of distinct entries in the data structure. The same for $D_v^L$, which has at most $2$ distinct entries. Thus the block at line~\ref{line:if_split} of \AlgoBuild$(S,i)$ runs in time $\scO(d|S| \log |S|)$.

If instead the condition at line~\ref{line:if_split} fails, then \AlgoBuild\ must compute $(j,a)$. To this end one proceeds as follows. First, for each $j \in [d]$ one computes the projection $S_{|j}$ of $S$ on the $j$-th feature (keeping the label as well). Then one sorts $S_{|j}$ according to the feature values in time $\scO(|S| \log |S|)$. Next, one scans $S_{|j}$ and finds the threshold $t^*$ for which a split on $j$ yields maximum gain in time $\scO(|S|)$. To this end one just needs to keep label counts for the subsequence formed by the first $i$ examples in $S_{|j}$, so that the gain a split at that point would yield can be computed in time $\scO(1)$ from the counts of the first $(i-1)$ examples. 
Summarizing, one can compute the optimal split $(j,a)$ in time $\scO(d |S| \log |S|)$. Since $|S[x_j \le a]|+|S[x_j > a]|=|S|$, it follows that \AlgoBuild$(S,i)$ always runs in time $O(d |S| \log |S| (h + \log |S|))$, which is in $O(d |S| \log n (h + \log n))$ since $|S| \le n$ by definition of $n$. For \AlgoUpdate, computing $ v_{\kappa_1},\dots, v_{\kappa_\ell} $ takes time $O(h)$, while performing any \ins\ or \del\ operation on $S_{v_{\kappa_\ell}}$ takes time $O(d \log |S^t|) = O(d \log n)$. Finally, computing the input $S_{v_{\kappa_j}}$ of \AlgoBuild\ takes time $O(|S_{v_{\kappa_j}}|)$ by visiting $T_{v_{\kappa_j}}$ and listing the data structures at its leaves.

Now, we bound the total time taken by $\scT$ successive invocations of \AlgoUpdate. Let $B = \{ t \in [\mathcal{T}\,]: \AlgoBuild\ \text{ is invoked at time } t\}$. For every $t \in B$ let $b(t)$ be such that $v_{b(t)}$ is the vertex $v_{\kappa_j}$ on which \AlgoBuild\ is invoked. The total running time $\cost(\mathcal T)$ of the $\mathcal T$ invocations satisfies:
\begin{align}
    \cost(\mathcal T) &\le \sum_{t=1}^{\mathcal T} O(h + d \log n)\nonumber \\&~~+ \sum_{t\in B} O\!\left((h+\log n) \cdot \log n \cdot d\cdot |S^t_{v_{b(t)}}|\right) \label{eq:total_cost} 
\end{align}
The first term contributes $O(\scT (h+d \log n))$. We now bound the second term. For every $t \in B$ consider the $t$-th execution of \AlgoUpdate.
Let $s^t(v)$ be the value of $s(v)$ right before \AlgoBuild\ is invoked, and $v_{i(t)}$ be the vertex that satisfies the condition at line~\ref{line:if_c_eps_s} of \AlgoUpdate. Note that $v_{i(t)}$ is by construction a descendant of $v_{b(t)}$. Finally, for $v \in \{v_{b(t)},v_{i(t)}\}$ let $c^t({v})$ and $s^t(v)$ be the values of $c(v)$ and $s(v)$ right after line~\ref{line:inc_c} is executed with $v_{\kappa_i}=v$. Then:
\begin{align}
\sum_{t \in B}  |S^t_{v_{b(t)}}|
& \leq   2 \cdot \sum_ {t \in B}  s^t(v_{b(t)}) \label{eq:2} \\
& \leq  4 \cdot \sum_ {t \in B}  s^t(v_{i(t)}) \label{eq:3} \\
& \leq  \frac{4}{\epsilon} \cdot \sum_ {t \in B}  c^t(v_{i(t)}) \label{eq:4} \\
& \leq  \frac{4}{\epsilon} \cdot \sum_ {t \in B}  c^{t}(v_{b(t)}) \label{eq:5}
\end{align}
where~\eqref{eq:2} follows from Lemma~\ref{lem:counter} noting that $\epsilon \le 1$,~\eqref{eq:3} and~\eqref{eq:4} follow respectively from lines~\ref{line:hat_s}-\ref{line:min_j} and line~\ref{line:if_c_eps_s} of \AlgoUpdate, and~\eqref{eq:5} follows from the fact that $c^t(v) \leq c^t(u)$ if $v$ is a descendant of $u$. Now observe that at every time $t$ at most $h$ counters $c(v)$ are increased by one unit; therefore, 
\begin{equation}
\sum_{t\in B} c^t({v_{b(t)}}) \leq |B| \cdot h \leq \mathcal{T} \cdot h
\end{equation}
We conclude that $\sum_{t \in B}  |S^t_{v_{b(t)}}| \le \mathcal{T} \frac{4 h}{\epsilon} $. Plugging this bound in~\eqref{eq:total_cost} and noting that the second sum dominates, we obtain:
\begin{equation}\label{eq:cost_1}
    \cost(\mathcal T) = O\!\left(\scT \cdot (h + \log n) \cdot \log n \cdot d \cdot  \frac{h}{\epsilon} \right) 
\end{equation}
Next we prove a second bound on $\cost(\mathcal T)$; the final bound comes from taking the minimum.

Consider the $t$-th execution of \AlgoUpdate, let $(x,y)$ the example that is inserted or removed, and recall that $P_{v(x)}$ is the path from the root of the tree to the leaf $v(x)$ determined by $x$. Let $C^{t} = \{v_{k_1} ,\dots, v_{k_M}\}$ be the set of all vertices of $P_{v(x)}$ such that \AlgoBuild$(S^{t_i}_{v_{k_i}}, i)$ is executed at some time $t_i \geq t$.

If $C^t \ne \emptyset$ then we call $C^t$ a \emph{charging set}. We wish to bound the maximum size of $C^t$, which might be seen as the number of \AlgoBuild\ operations performed per update. We shall prove the following two properties:
\begin{itemize}
\item[P1:]  $\forall t \in [\mathcal{T}]$, $|C^t| \leq \lceil \log ( n )  \rceil$
\item[P2:] $\forall t \in B$, $c^t(v_{b(t)}) \leq \sum_{\tau=1}^{t} \Ind_{v_{b(t)} \in C^{\tau}}$
\end{itemize}
We start with P1. We argue that there cannot be distinct nodes $v_{k_i}$ and $v_{k_j}$ in $C^{t}$ such that $ \lceil \log s^{\tau}(v_{k_i}) \rceil   = \lceil \log s^{\tau}(v_{k_j}) \rceil$ for any $\tau \in [t_i,t_j]$.
Suppose $\tau=t_i$; $v_{k_i}$ cannot be an ancestor of $v_{k_j}$, for otherwise $v_{k_j}$ would not be connected to the root node at time $t_j$ and \AlgoBuild$(S^{t_j}_{v_{k_j}},j)$ would not be executed. If $v_{k_j}$ is an ancestor of $v_{k_i}$, \AlgoUpdate\ would not have performed \AlgoBuild$(S^{t_i}_{ v_{k_i}},i)$ at time $t_i$, in that, there is at least one other node ($v_{k_j}$) which is closer to the root and that would have been selected instead (line~\ref{line:min_j} of \AlgoUpdate). The claim holds for every $\tau$, as $s^{\tau} (v_{k_j})=s^{t_i} (v_{k_j})$, for every $\tau \in [t_i,t_j]$. Therefore $|C^t| \leq \lceil \log n  \rceil$.  

For P2 we proceed as follows. Let $t_0 < t$ be such that $c(v_{b(t)})$ is set to $0$ by \AlgoBuild\ (i.e., the point in time when $v_{b(t)}$ was created by \AlgoBuild, line~\ref{line:creation}), and let $q=c^t(v_{b(t)})$. By construction of \AlgoUpdate, there are $q$ distinct times $t_0 +1 \le \tau_1 < \dots < \tau_q \leq t$ such that, for every $i \in [q]$, we have $c^{\tau_{i}}(v_{b(t)}) = c^{\tau_{i-1}}(v_{b(t)})+1$ and $v_{b(t)} \in C^{\tau_i}$, proving P2.

We obtain the following chain of inequalities:
\begin{align}
    \sum_{t \in B} c^t(v_{b(t)})
    &\le \sum _{t \in B} \sum_{\tau=1}^{t} \Ind_{v_{b(t)} \in C^{\tau}}
    \le \sum _{t \in B} |C^{\tau}|
\end{align}
where the inequalities follow respectively from P2 and from the fact that $v_{b(t)} \ne v_{b(t')}$ for $t \ne t'$. Using P1 and $B \subseteq [\mathcal T]$, we conclude that $\sum_{t \in B} c^t(v_{b(t)}) = O\!\left( \mathcal{T} \cdot \log n \right)$. Plugging this bound into~\eqref{eq:total_cost} yields:
\begin{equation}\label{eq:cost_2}
    \cost(\mathcal T) = O\!\left(  \scT \cdot (h + \log n) \cdot \log n \cdot d \cdot \frac{\log n}{\epsilon} \right) 
\end{equation}
Taking the minimum of~\eqref{eq:cost_1} and~\eqref{eq:cost_2} yields that  $\cost(\mathcal T)$ is in
\begin{align}
    \scO\!\left( \scT \cdot (h + \log n) \cdot \log n \cdot d \cdot  \frac{\min(h, \log n)}{\epsilon} \right)
\end{align}
As $(x+y)\min(x,y) \le 2 xy$ for $x,y \ge 0$ we conclude that $\cost(\mathcal T) = \scO\big(\scT d h (\log n)^2 / \epsilon \big)$. Noting that $h \in \scO\big(\frac{\log n}{\epsilon}\big)$ by Lemma~\ref{lem:log_height} concludes the proof.
\end{proof}

\section{Lower Bounds}
\label{sec:lb}
This section proves lower bounds on the space and amortized running time used by any fully dynamic algorithm for our problem. These bounds hold even under significant relaxations of both the input access model and the constraints of Section~\ref{sec:dyn}. Notably, they hold for randomized algorithms that can fail with constant probability under inputs provided by oblivious adversaries.

To state our bounds we need some more definitions. A label $y$ is $\beps$-feasible for $x \in \scX$ w.r.t.\ $S$ if there is a decision tree $(T,\Splits,L)$ that is $\beps$-feasible w.r.t.\ $S$ such that $T(x)=y$. Note that there might be multiple $\beps$-feasible labels for $x$.  A decision tree algorithm $\mathcal{A}$ is \emph{weakly $(\beps,\delta)$}-feasible w.r.t.\ $S$ if for every $x \in \scX$ there is a decision tree $(T_x,\Splits_x,L_x)$ that is $\beps$-feasible w.r.t.\ $S$ and such that $\Pr(\mathcal{A}_S(x)=T_x(x)) \ge \delta$. Note that $(\beps,\delta)$-feasibility is much weaker than $\beps$-feasibility: not only it allows the algorithm to fail, but it does not even require it to be coherent with any given $\beps$-feasible tree.

\begin{theorem}\label{thm:lb}
Let $k^*, h^* \ge 1$, and let $\beps=(\epsThr,\epsApx)$ with $0 \le \epsThr \le 1$ and $0 \le \epsApx < \frac1{24}$. Any weakly $(\beps,\frac34)$-feasible fully dynamic algorithm with pruning thresholds $k^*, h^*$ uses space $\Omega \big( \frac{n \cdot d} {k \cdot \log n} \big)$, where $d$ is the number of features and $n$ is the maximum size of the active set at any point in time. 
\end{theorem}
\begin{proof}
We reduce from the following classic two-party communication problem called \ind. Alice is given a string $x \in \{0,1\}^N$ and Bob is given an integer $i \in [N]$. Alice is allowed to send one message $\mathcal{M} \in \{0,1\}^*$ to Bob, which, after receiving $\mathcal M$, outputs a single bit. The goal of Bob is to output precisely $x_i$. It is well known that for Bob to succeed with probability greater than $\frac34$ we must have $|\mathcal{M}|=\Omega(N)$, see~\cite{DBLP:conf/esa/HenzingerK20}.

We reduce \ind\ to the construction of an $(\beps,\frac34)$-feasible fully dynamic algorithm. For some positive integers $N,D$, Alice is given an arbitrary string in $\{0,1\}^{ND}$ representing a matrix $A \in \{0,1\}^{N \times D}$. Bob is given a pair $(\kappa,\ell) \in [N] \times [D]$ and must output $A_{\kappa\ell}$. By the lower bound above, Alice must send to Bob $\Omega(ND)$ bits in order for Bob to succeed with probability greater than $\frac34$. 

The reduction is as follows. Let $k=k^*$. First, Alice computes the following sequence $S$ of $|S|=N \cdot D \cdot 2k$ examples. Let $\bar{D} := \lceil \log (N + D ) + 1 \rceil $, and for all $i \in [N+D]$ let $\pmb{b}_{i} \in \{0,1\}^{\bar D}$ be the binary representation of $i$. For simplicity and w.l.g. we assume $k$ to be an even integer. For every $i \in [N+D]$ Alice constructs $2k$ examples $(\bx_{i}^1,y_{i}^1) ,\dots,(\bx_{i}^{2k},y_{i}^{2k})$ with $\scX = \{0,1\}^{D+\bar D}$ and $\scY=\{0,1\}$, as follows. For every $i \in [N+D]$ and every $h \in [2k]$, the last $\bar D$ bits of $\bx_i^h$ correspond to the string $\pmb{b}_i$ (i.e., $x_{ij}^h = b_{ij}$ for all $j \in [D+1,D+\bar D]$), and $y_i^h = \Ind_{h > k}$. The remaining bits of $\bx_i^h$ are defined as follows. If $i \in [N]$, then for all $j \in [D]$:
\[
x_{ij}^h:=
\begin{cases}
1-A_{ij},& h \in [k];  \\
A_{ij},& h \in [k+1,2k];
\end{cases}
\]
while if $i \in [N+1,N+D]$, then for all $j \in [D]$ and $h \in [2k]$:
\[
x_{ij}^h:=
\begin{cases}
1, & j \in [D] \setminus \{i-N\}, h \textrm{ mod } 2 = 0; \\
0, & \textrm {otherwise;}
\end{cases}
\]
Let $\mathcal{A}$ be any $(\beps,\frac34)$-feasible fully dynamic algorithm with $\epsApx < \frac1{24}$. Alice asks $\mathcal{A}$ to add every element of $S$, then she sends a snapshot of its memory to Bob, which resumes the execution of $\scA$. Next, Bob asks $\scA$ to perform $\del(\bx, y)$ for every $y \in \{0,1\}$ and every $\bx \in \{0,1\}^{D+\bar D}$ terminating with the $\bar D$-bits binary string $\pmb{b}_i$, for all $i \in [N + D] \setminus \{\kappa, N + \ell\}$. Finally, Bob asks $\mathcal{A}$ to label $\pmb{1}$, and outputs the answer.

First, we claim that Bob outputs $A_{\kappa\ell}$. To prove this, note that the active set received by $\scA$ is:
\begin{align*}
    \widehat S = \left\{(\bx_{ \kappa }^h,y_{ \kappa}^h) : h \in [k] \right\} \cup \left\{(\bx_{ N+\ell }^h, y_{N +\ell}^h) : h \in [k]\right\} 
\end{align*}
We prove that any decision tree that is $\beps$-feasible with respect to $\widehat S$ labels the example $\pmb{1}$ with $A_{\kappa \ell}$. To this end we show that in any such tree, (i) the root splits on feature $\ell$, (ii) the child $v$ of the root corresponding to feature $\ell$ equal to $1$ is a leaf with label $A_{\kappa\ell}$. For (i), we prove that $\widehat S$ does not meet any stopping condition, and that $j$ is the only $\epsApx$-optimal feature. The claim on the stopping condition is immediate. For the optimality of $\ell$, we claim that:
\begin{align}
    G(\widehat S, j) =
\begin{cases}
\nicefrac{1}{6},& j=\ell\\
\nicefrac{1}{8},& j \in [D] \setminus \{\ell\}\\
0, &\text{otherwise}
\end{cases}
\end{align}
To begin, note that $g(\widehat S)=\frac{1}{2}$. Let $\widehat S_1$ and $\widehat S_2$ be the two subsequences obtained by splitting $\widehat S$ on $j$. When $j \in \ell$, $\widehat S_1$ contains $k$ examples with identical labels, so $g(\widehat S_1)=0$, while $\widehat S_2$ contains $3k$ examples of which $2k$ have identical label, so $g(\widehat S_2)=\frac{4}{9}$. Thus $G(\widehat S,\ell) = \frac{1}{2} - \left(\frac{1}{4}\cdot 0 + \frac{3}{4} \cdot \frac{4}{9}\right) = \frac{1}{6}$.
When $j \in [D] \setminus \{\ell\}$, both $\widehat S_1$ and $\widehat S_2$ contain $k$ examples of $\left\{(\bx_{ \kappa }^h,y_{ \kappa}^h) : h \in [k] \right\}$ with identical label, as well as $k$ examples of $\left\{(\bx_{N+\ell}^h,y_{N+\ell}^h) : h \in [k]\right\}$ with $\frac{k}{2}$ labels to $0$ and $\frac{k}{2}$ labels to $1$. Thus, in both $\widehat S_1$ and $\widehat S_2$ one label occurs precisely on a fraction $\frac{3}{4}$ of the examples. Hence $g(\widehat S_1)=g(\widehat S_2)=2 \cdot \frac{3}{4} \cdot \frac{1}{4} = \frac{3}{8}$, and $G(\widehat S,j) = \frac{1}{8}$. 
In every other case, either $\widehat S_1=\widehat S$, or $\widehat S_1=\left\{(\bx_{ \kappa }^h,y_{ \kappa}^h) : h \in [2k] \right\}$ and $\widehat S_2=\left\{(\bx_{N+\ell}^h, y_{N+\ell}^h) : h \in [2k]\right\}$, which implies $g(\widehat S_1)=g(\widehat S_2)=\frac{1}{2}$ and $G(\widehat S,j)=0$.
Since $\epsApx < \frac{1}{24} = \frac{1}{6}-\frac{1}{8}$, we conclude that $\ell$ is the only $\epsApx$-optimal feature, as desired.

For (ii), note that the subsequence of $\widehat{S}$ having the $\ell$-th feature set to $1$ has all labels equal to $A_{\kappa\ell}$. This implies that the corresponding child $v$ of the root is a leaf, since it meets at least one of the stopping conditions, and that it assigns label $A_{\kappa\ell}$. This proves that Bob returns $A_{\kappa\ell}$.

To prove the space lower bound, note that $S$ consists of $n=2k(N+D)$ examples, each of which can be encoded in $d=D+\bar D = O(D + \log(D+N))$ bits. For $D = O(N)$, this yields $n=O(kN)$ and $d=O(D \log N)$ and therefore $ND = \Omega\big(\frac{nd}{k\log n}\big)$. Recalling that Alice must send $\Omega(ND)$ bits to Bob concludes the proof.
\end{proof}

We conclude this section with a lower bound on the running time of any fully dynamic algorithm. Clearly, if the model requires the algorithm to read every labeled example $(x,y)$ upon arrival, then a lower bound of $\Omega(nd)$ is trivial. However, we show that an $\Omega(nd)$ bounds holds even if we do not require the algorithm to read the examples; instead, at any point in time we allow the algorithm  to access in time $\scO(1)$ the $j$-th feature of \emph{any} example in the current active set. We call this the \emph{matrix access model}. Again, we prove the bound for weakly $(\beps,\delta)$-feasible algorithms.
\begin{theorem}\label{thm:time_lb}
Let $k,h \ge 1$ and $\epsThr,\epsApx\in [0, \frac{1}{2})$. For arbitrarily large $n$ and $d$ there exist sequences of $n$ \ins\ and \del\ operations over $\{0,1\}^d \times \{0,1\}$ such that, in the matrix access model, any weakly $(\beps,\nicefrac{2}{3})$-feasible fully dynamic algorithm has expected running time $\Omega(nd)$.
\end{theorem}

\section{Experiments}
We compare \algo\ against two state-of-the-art algorithms for incremental decision tree learning, EFDT \cite{Manapragada18} and HAT \cite{10.1007/978-3-642-03915-7_22}, using the MOA software \cite{DBLP:journals/jmlr/BifetHKP10}. Similarly to \algo, EFDT and HAT aim at keeping $\epsilon$-optimal splits, which they do with high probability when the examples are i.i.d. from a distribution. The results for HAT and other experiments can be found in Appendix~\ref{sec:addexp}.

\noindent\textbf{Settings.} We implemented \algo\ in C++; the source code is available as supplemental material and will be released as open-source. We conducted all experiments on an Ubuntu 20.04.2 LTS server equipped with 144 Intel(R) Xeon(R) Gold 6154 @ 3.00GHz CPUs and 264 GB of RAM. We observe that the algorithms have not been implemented in the same programming language, which limits the relevance of the runtime comparison.

\noindent\textbf{Datasets.} Our datasets are shown in Table~\ref{tab:datasets_table}. We have chosen them among standard datasets for classification; some of them, such as INSECTS, feature the so-called concept drift. Not all datasets have binary labels. For the INSECTS datasets, we assigned label $1$ to  the union of \texttt{male} classes. For every other dataset, we assigned label $1$ to the majority class.

\noindent\textbf{Input models.} We consider three input models. Let $(x_1,y_1), \dots ,(x_T,y_T)$ be the sequence of examples as given by the dataset at hand (typically in chronological order). The simplest model is when only insertions are allowed aka \emph{incremental} model. 
Formally, at every $t \in [T]$ the algorithm receives $\ins(x_t,y_t)$. This model is supported by all algorithms (\algo, EDFT, HAT), hence we use it to compare them against each other.
The next two models involve deletions and thus are supported only by \algo. The first one is the \emph{sliding window} model (SW): given an integer $W \ge 1$ for all $t \in [T]$ the algorithm receives $\ins(x_t,y_t)$, preceded by $\del(x_{t-W+1},y_{t-W+1})$ if $t \ge W$. The second one is the \emph{random update} model (RU): for all $t \in [T]$, with probability $\nicefrac12$ the algorithm receives $\ins(x_t,y_t)$ and with probability $\nicefrac12$ it receives $\del(x,y)$ where $(x,y)$ is chosen uniformly at random from the active set $S^t$. 

\noindent\textbf{Metrics.} As is customary in the literature, we evaluate how well each algorithm predicts the label of the next example before ``seeing'' it. Formally, if $(x_t,y_t)$ is the $t$-th example appearing in the input sequence, then we compute $\hat y_t = \lab(x_t)$ \emph{before} the algorithm sees $(x_t,y_t)$. We then compute the F1-score of the label sequence $\hat\by=(\hat y_t)_{t \ge 1}$ against the ground-truth $\by=(y_t)_{t \ge 1}$,
\begin{align}
    \text{F1}(\by,\hat\by) = \frac{2 \cdot P(\by,\hat\by)\cdot R(\by,\hat\by)}{P(\by,\hat\by)+R(\by,\hat\by)}
\end{align}
where $P(\by,\hat\by)$ and $P(\by,\hat\by)$ are respectively the precision and recall of $\hat\by$ against $\by$. The F1-score is in $[0,1]$ with higher values denoting better results.  

\noindent\textbf{Parameters.} 
For \algo, we let $\epsThr = 0$, $\epsApx = 0$ $k = 1$, $h \in \{5,10\}$, and we manually set $\epsilon \in [0,2]$. Note that it breaks the condition of theorem \ref{thm:main_UB}. It allows us to test the effect of $\epsilon$ without fine-tuning of the other parameters.
The parameters of EFDT and HAT are set to the original values specified by the authors; we only vary the so-called grace period in $\{100, 500, 1000\}$ to find the value yielding highest F1-score. For the SW model we use $W\in\{100,1000\}$.  In all our experiments, we first build a decision tree for the first $W$ examples, then we apply the models above to the remaining sequence. Several parameter configurations show similar trends. We only report the most interesting results here; see Appendix~\ref{sec:addexp} for more.

\noindent \textbf{\algo\ versus EFDT.} We compare the F1-scores of EFDT and \algo\ when allowed the same amortized time. To this end we tuned \algo's $\epsilon$ to make its running time very close to (and never exceeding) that of EFDT. The results are shown in Table~\ref{tab:EFDT_FDDT}; remarkably, \algo\ outperforms consistently EFDT in terms of F1-score. One of the possible reasons is that \algo\ can guarantee to be relatively close to the optimal Gini gain, even without computing it explicitly. In contrast, EFDT resorts to an approximation which might be relatively poor, given that maintaining an optimal Gini gain is expensive. 

\noindent \textbf{\algo\ on SW and RU.} Next, we studied the performance of \algo\ in the SW and RU models (recall that EFDT/HAT do not work here). For the SW model, we set $h=10$, $k=1$, $\epsThr=0$, and $W=100$ for Electricity and $W=1000$ otherwise. Figure~\ref{fig:SW} shows the F1 score as a function of $\epsilon$ (subfigures a-c) and the average time per update in milliseconds in logarithmic scale as a function of $\epsilon$ (subfigures d-f). The smaller $\epsilon$ is, the more often subtrees are recomputed, yielding a higher F1 score and amortized running time. This behavior is clear in the Electricity and Poker datasets, where from $\epsilon=0$ to $\epsilon=1$ the F1 score decreases by roughly $0.1$ and the running time increases by three orders of magnitude. A good tradeoff could be $\epsilon=0.1$, where the F1-score is close to that of $\epsilon=0$ but with an amortized running time per update smaller by orders of magnitude ($\approx 0.5$ms). For INSECTS the F1-score is much more stable. All other datasets and parameter settings yielded very similar qualitative behaviors. Figure~\ref{fig:RU} shows the average running time for the RU model, showing similar trends to Figure~\ref{fig:SW}. 

\begin{table}
    \centering
    \begin{tabular}{llll}
        \toprule
         & $d$ & \# of examples & 1-class \\
        \midrule
        Electricity & 8 & \numprint{45311} & UP \\
        Forest Covertype & 54 & \numprint{581011} & 2 \\
        INSECTS v1-v5 & 33 & \numprint{24150} -- \numprint{79986} & *-male \\
        KDDCUP99 & 41 & \numprint{494021} & smurf. \\
        NOAA Weather & 8 & \numprint{18159} & 1 \\
        Poker & 10 & \numprint{829201} & 0 \\
        \bottomrule
    \end{tabular}
    \caption{Datasets statistics.}
    \label{tab:datasets_table}
\end{table}

\begin{table}
    \centering
    \begin{tabular}{llllll}
        \toprule
         & \multicolumn{2}{c}{EFDT} & \multicolumn{2}{c}{\algo} & \\
         & RT & F1 & RT & F1 & $\epsilon$ \\
        \midrule
        Electricity & 1.65 & 83.64 & 1.53 & \textbf{90.33} & 0.15 \\
        Forest Covertype & 42.47 & 83.64 & 42.37 & \textbf{90.33} & 0.29 \\
        INSECTS v1 & 4.85 & 88.96 & 4.51 & \textbf{92.17} & 1.00 \\ 
        INSECTS v2 & 3.13 & 87.40 & 3.09 & \textbf{92.53} & 0.92 \\
        INSECTS v3 & 7.54 & 92.51 & 7.43 & \textbf{94.76} & 1.00 \\
        INSECTS v4 & 6.30 & 91.15 & 6.02 & \textbf{91.91} & 0.95 \\
        INSECTS v5 & 6.84 & 89.85 & 6.77 & \textbf{93.34} & 1.00 \\ 
        KDDCUP99 & 17.72 & 97.98 & 17.43 & \textbf{99.91} & 0.17 \\
        NOAA Weather & 0.73 & 80.78 & 0.73 & \textbf{81.43} & 0.36 \\
        Poker & 16.26 & 79.69 & 16.14 & \textbf{86.07} & 1.03 \\
        \bottomrule
    \end{tabular}
    \caption{Running time in seconds (labeled \textit{RT}) and F1-score (labeled \textit{F1}) of EFDT and \algo\ in the incremental model. The last column shows the value of $\epsilon$ in \AlgoUpdate.}
    \label{tab:EFDT_FDDT}
\end{table}

\newcommand{\mylen}{.4\textwidth}
\begin{figure}[h!] 
    \centering\small
    \begin{subfigure}[b]{\mylen}
         \centering
         \includegraphics[width=\textwidth]{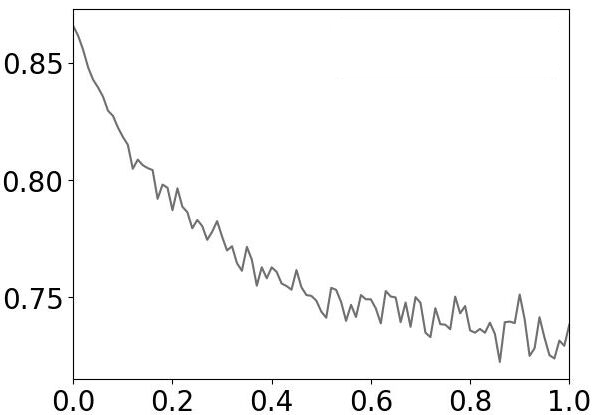}
         \label{fig:f1_elec}    
    \end{subfigure}~~
    \begin{subfigure}[b]{\mylen}
         \centering
         \includegraphics[width=\textwidth]{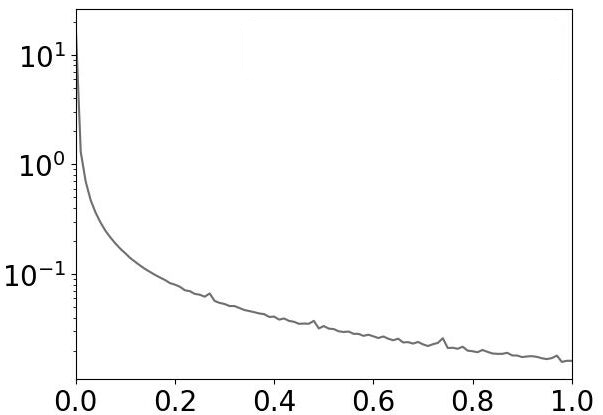}
         \label{fig:iter_elec}    
    \end{subfigure}
    \\
    \begin{subfigure}[b]{\mylen}
         \centering
         \includegraphics[width=\textwidth]{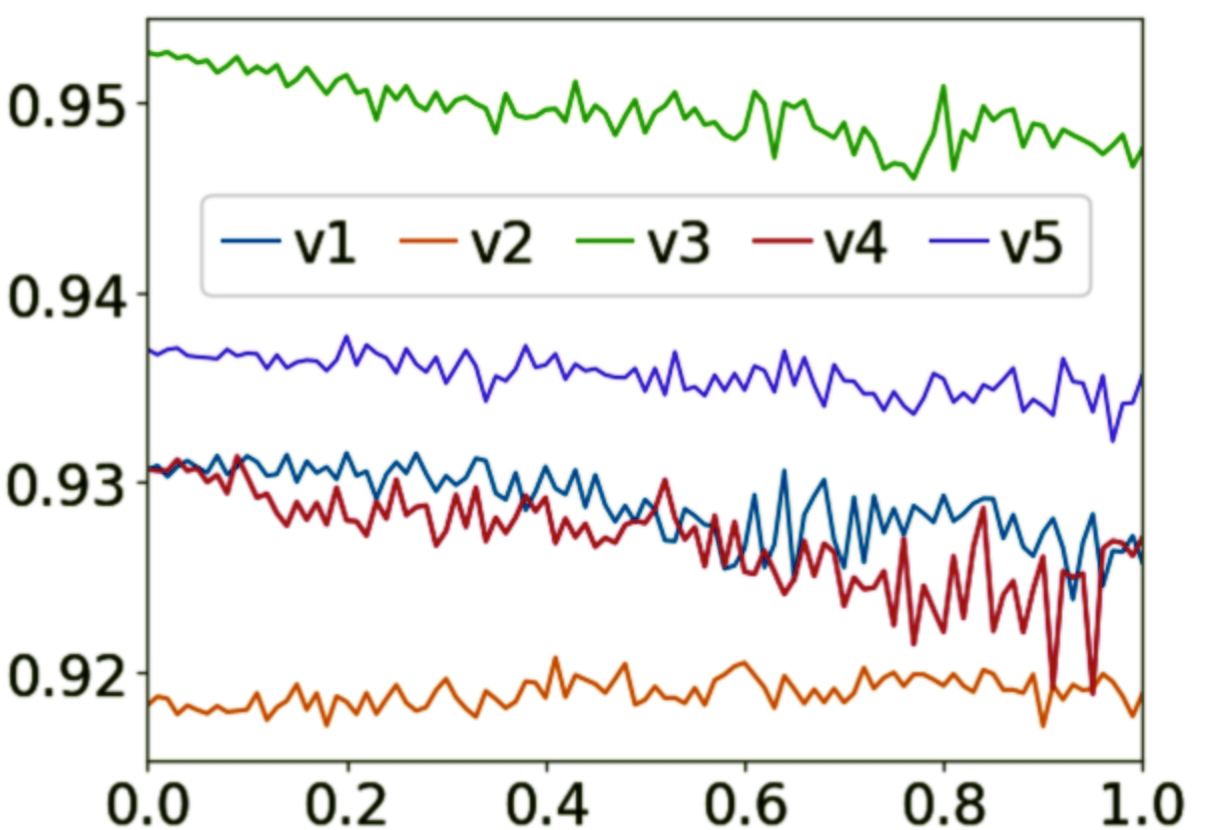}
         \label{fig:f1_insects}    
    \end{subfigure}~~
    \begin{subfigure}[b]{\mylen}
         \centering
         \includegraphics[width=\textwidth]{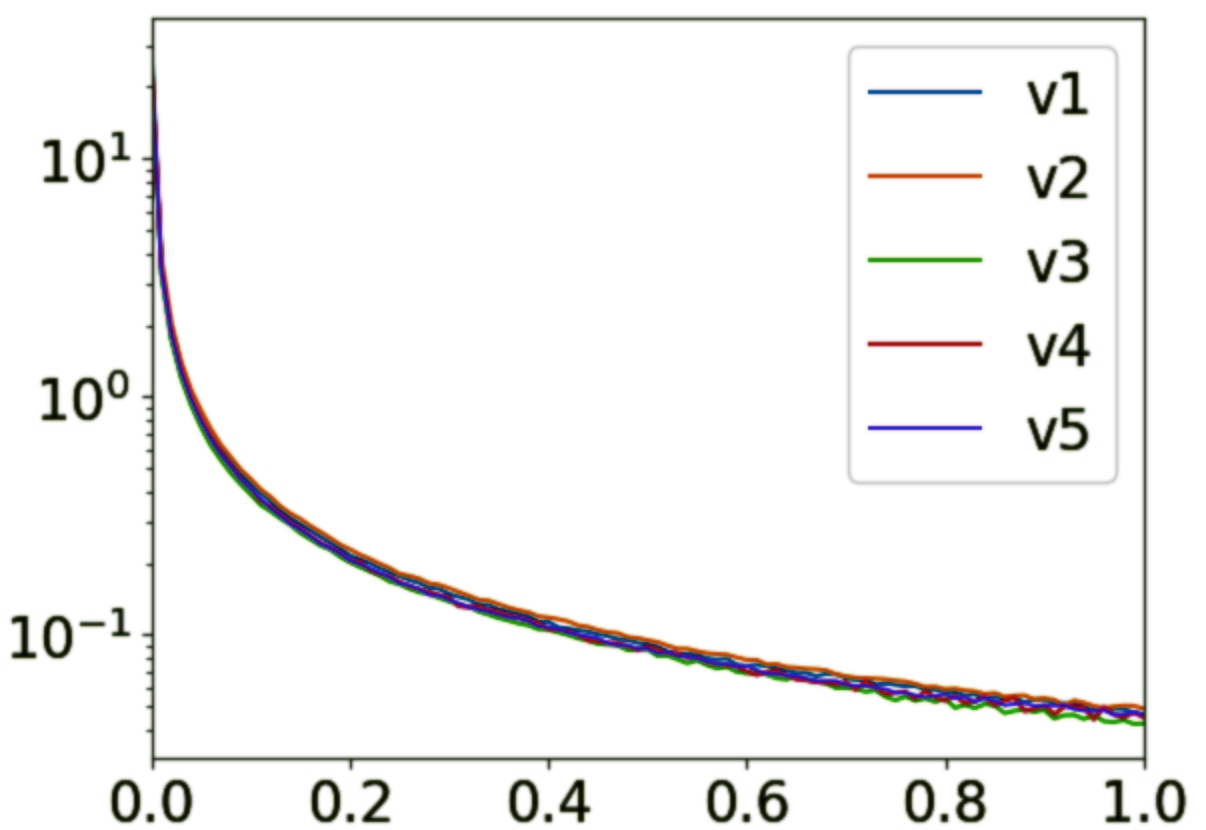}
         \label{fig:iter_insects}    
    \end{subfigure}
    \\
    \begin{subfigure}[b]{\mylen}
         \centering
         \includegraphics[width=\textwidth]{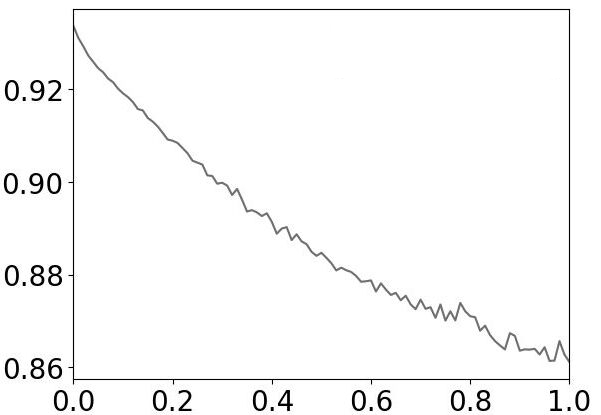}
         \label{fig:f1_poker}    
    \end{subfigure}~~
    \begin{subfigure}[b]{\mylen}
         \centering
         \includegraphics[width=\textwidth]{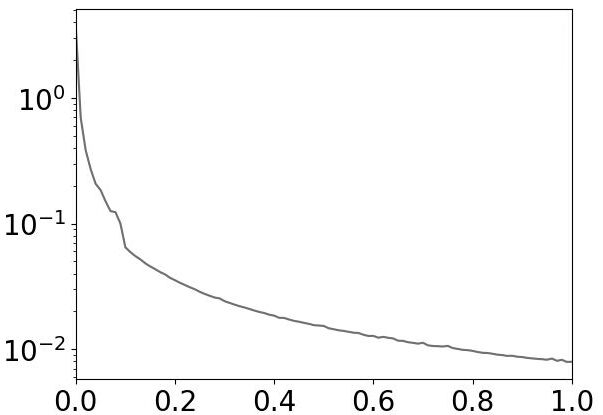}
         \label{fig:iter_poker}    
    \end{subfigure}
\caption{Performance of \algo\ in the SW model on the Electricity, INSECTS and Poker datasets (top to bottom), in terms of F1-score (left) and amortized milliseconds per update (right) as a function of $\epsilon$.}\label{fig:SW}
\vspace*{10pt}
\end{figure}
\begin{figure}[h!]
\vspace*{10pt}
    \centering
    \hfill
    \begin{subfigure}[b]{\mylen}
         \centering
         \includegraphics[width=\textwidth]{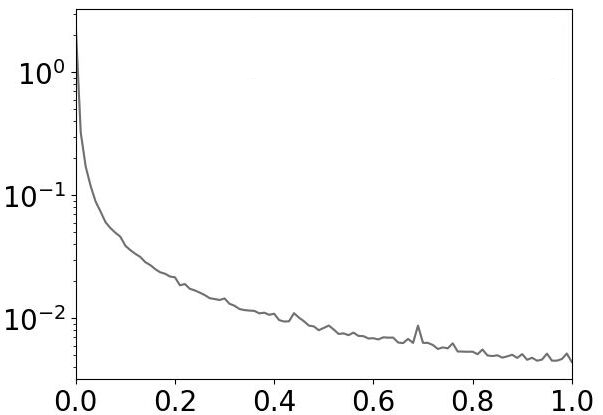}
         \label{fig:ru_elec}    
    \end{subfigure}
    \hfill
    \begin{subfigure}[b]{\mylen}
         \centering
         \includegraphics[width=\textwidth]{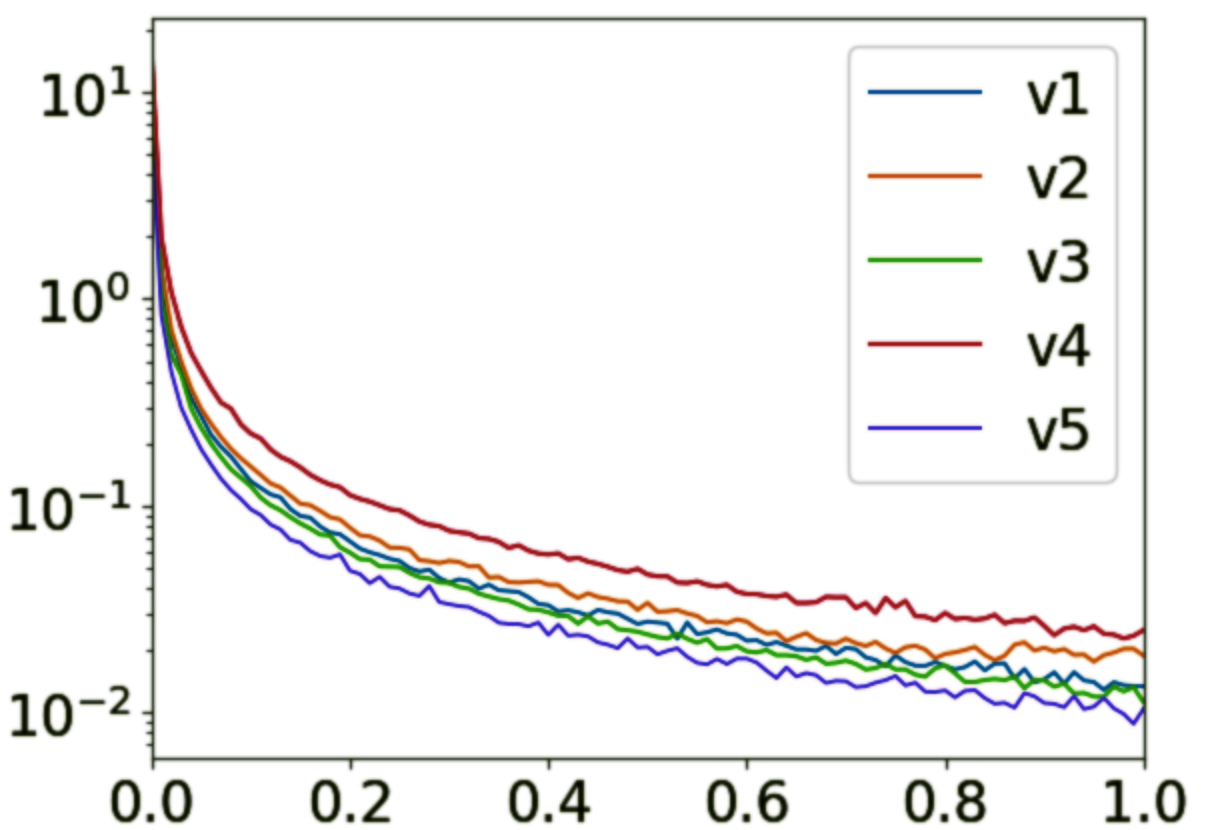}
         \label{fig:ru_insects}    
    \end{subfigure}
    \hfill\\
\caption{Amortized running time of \algo\ in the RU model on the Electricity (left) and INSECTS (right) datasets. Similar results are obtained for the Poker dataset, which are omitted for space constraints.  }
\label{fig:RU}
\end{figure}

\section{Conclusions and Future Work}
We developed the first fully dynamic algorithm for maintaining $\epsilon$-feasible decision trees, and proved it to be nearly optimal in terms of space and amortized time. Our work shows that many well-known decision tree algorithms, whether offline like CART or incremental like EDFT, can be made fully dynamic with a small loss in the quality of the decision tree and a small overhead in the amortized running time. Our work leaves open the natural question of whether these results can be strengthened from amortized to worst-case. We believe this is an exciting direction for future research in fully-dynamic algorithms for supervised machine learning.

\section*{Acknowledgements}
The work of Gabriel Damay and Mauro Sozio  was partially supported by the French National Agency (ANR) under project APY (ANR-20-CE38-0011), while it has been carried out partially in the frame of a cooperation between Huawei Technologies France SASU and Telecom Paris (Grant no. YBN2018125164). Marco Bressan has been supported in part by a Google Focused Research Award ``Algorithms and Learning for AI'' (ALL4AI). We thank the anonymous reviewers for their careful reading of our manuscript and their many insightful comments and suggestions.

\bibliography{biblio.bib}

\begin{thebibliography}{10}

\bibitem{DBLP:journals/corr/abs-2112-07050}
MohammadHossein Bateni, Hossein Esfandiari, Rajesh Jayaram, and Vahab~S.
  Mirrokni.
\newblock Optimal fully dynamic k-centers clustering.
\newblock {\em CoRR}, abs/2112.07050, 2021.

\bibitem{BENTLEY1980301}
Jon~Louis Bentley and James~B Saxe.
\newblock Decomposable searching problems i. static-to-dynamic transformation.
\newblock {\em Journal of Algorithms}, 1(4):301--358, 1980.

\bibitem{DBLP:conf/stoc/BhattacharyaHNT15}
Sayan Bhattacharya, Monika Henzinger, Danupon Nanongkai, and Charalampos~E.
  Tsourakakis.
\newblock Space- and time-efficient algorithm for maintaining dense subgraphs
  on one-pass dynamic streams.
\newblock In {\em Proc.\ of ACM STOC}, 2015.

\bibitem{10.1007/978-3-642-03915-7_22}
Albert Bifet and Ricard Gavald\`{a}.
\newblock Adaptive learning from evolving data streams.
\newblock In {\em Proc.\ of IDA}, page 249–260, Berlin, Heidelberg, 2009.
  Springer-Verlag.

\bibitem{DBLP:journals/jmlr/BifetHKP10}
Albert Bifet, Geoff Holmes, Richard Kirkby, and Bernhard Pfahringer.
\newblock {MOA:} massive online analysis.
\newblock {\em J. Mach. Learn. Res.}, 11:1601--1604, 2010.

\bibitem{DBLP:conf/www/ChanGS18}
T.{-}H.~Hubert Chan, Arnaud Guerqin, and Mauro Sozio.
\newblock Fully dynamic \emph{k}-center clustering.
\newblock In {\em Proc.\ of WWW}, 2018.

\bibitem{DBLP:conf/nips/Cohen-AddadHPSS19}
Vincent Cohen{-}Addad, Niklas Hjuler, Nikos Parotsidis, David Saulpic, and
  Chris Schwiegelshohn.
\newblock Fully dynamic consistent facility location.
\newblock In {\em Proc.\ of NeurIPS}, 2019.

\bibitem{Das19}
Ariyam Das, Jin Wang, Sahil~M. Gandhi, Jae Lee, Wei Wang, and Carlo Zaniolo.
\newblock Learn smart with less: Building better online decision trees with
  fewer training examples.
\newblock In {\em Proc.\ of IJCAI}, pages 2209--2215, 7 2019.

\bibitem{DBLP:journals/tkdd/StefaniERU17}
Lorenzo De~Stefani, Alessandro Epasto, Matteo Riondato, and Eli Upfal.
\newblock Tri{\`{e}}st: Counting local and global triangles in fully dynamic
  streams with fixed memory size.
\newblock {\em {ACM} Trans. Knowl. Discov. Data}, 2017.

\bibitem{Domingos00-HighSpeedStreams}
Pedro Domingos and Geoff Hulten.
\newblock Mining high-speed data streams.
\newblock In {\em Proc.\ of ACM KDD}, page 71–80, 2000.

\bibitem{DBLP:conf/www/EpastoLS15}
Alessandro Epasto, Silvio Lattanzi, and Mauro Sozio.
\newblock Efficient densest subgraph computation in evolving graphs.
\newblock In {\em Proc.\ of WWW}, 2015.

\bibitem{Gama03-HighSpeedStreams}
Jo\~{a}o Gama, Ricardo Rocha, and Pedro Medas.
\newblock Accurate decision trees for mining high-speed data streams.
\newblock In {\em Proc.\ of ACM KDD}, page 523–528, New York, NY, USA, 2003.
  Association for Computing Machinery.

\bibitem{Haug22}
Johannes Haug, Klaus Broelemann, and Gjergji Kasneci.
\newblock Dynamic model tree for interpretable data stream learning, 2022.

\bibitem{DBLP:conf/esa/HenzingerK20}
Monika Henzinger and Sagar Kale.
\newblock Fully-dynamic coresets.
\newblock In Fabrizio Grandoni, Grzegorz Herman, and Peter Sanders, editors,
  {\em Proc.\ of ESA}, volume 173 of {\em LIPIcs}, pages 57:1--57:21, 2020.

\bibitem{Domingos01-MiningTimeSeries}
Geoff Hulten, Laurie Spencer, and Pedro Domingos.
\newblock Mining time-changing data streams.
\newblock In {\em Proc.\ of ACM KDD}, page 97–106, New York, NY, USA, 2001.
  Association for Computing Machinery.

\bibitem{Jin03-Streaming}
Ruoming Jin and Gagan Agrawal.
\newblock Efficient decision tree construction on streaming data.
\newblock In {\em Proc.\ of ACM KDD}, page 571–576, New York, NY, USA, 2003.
  Association for Computing Machinery.

\bibitem{Manapragada2020}
Chaitanya Manapragada, Heitor~M. Gomes, Mahsa Salehi, Albert Bifet, and
  Geoffrey~I. Webb.
\newblock An eager splitting strategy for online decision trees in ensembles.
\newblock {\em Data Mining and Knowledge Discovery}, 36(2):566--619, 2022.

\bibitem{Manapragada18}
Chaitanya Manapragada, Geoffrey~I. Webb, and Mahsa Salehi.
\newblock Extremely fast decision tree.
\newblock In {\em Proc.\ of ACM KDD}, page 1953–1962, New York, NY, USA,
  2018.

\bibitem{Rutkowski13-MiningStreams}
Leszek Rutkowski, Lena Pietruczuk, Piotr Duda, and Maciej Jaworski.
\newblock Decision trees for mining data streams based on the mcdiarmid's
  bound.
\newblock {\em IEEE Transactions on Knowledge and Data Engineering},
  25(6):1272--1279, 2013.

\bibitem{DBLP:conf/stoc/SawlaniW20}
Saurabh Sawlani and Junxing Wang.
\newblock Near-optimal fully dynamic densest subgraph.
\newblock In Konstantin Makarychev, Yury Makarychev, Madhur Tulsiani, Gautam
  Kamath, and Julia Chuzhoy, editors, {\em Proceedings of the 52nd Annual {ACM}
  {SIGACT} Symposium on Theory of Computing, {STOC}}, pages 181--193. {ACM},
  2020.

\bibitem{shalevshwartz2014understanding}
Shai Shalev-Shwartz and Shai Ben-David.
\newblock {\em Understanding Machine Learning: From Theory to Algorithms}.
\newblock Cambridge University Press, USA, 2014.

\bibitem{ijcai2020-177}
Jian Sun, Hongyu Jia, Bo~Hu, Xiao Huang, Hao Zhang, Hai Wan, and Xibin Zhao.
\newblock Speeding up very fast decision tree with low computational cost.
\newblock In {\em Proc.\ of IJCAI}, pages 1272--1278, 7 2020.

\end{thebibliography}

\clearpage
\appendix
\section{Appendix}

\subsection{Supplementary material for Section~\ref{sec:prelim}}\label{apx:gini}
We start with two ancillary results used in the proof of Lemma~\ref{lem:smoothness} below.
\begin{lemma}
\label{lem:smooth_gini}
Let $S,S' \in (\scX \times \{0,1\})^{\ge 1}$. If $\ED(S,S') \le 1$ then $|g(S)-g(S')| < \frac{2}{\max(|S|,|S'|)}$.
\end{lemma}
\begin{proof}
Without loss of generality assume $|S|\ge|S'|$, let $n=|S|=\max(|S|,|S'|)$, and write $S=(\bx_1,y_1),\ldots,(\bx_n,y_n)$.
Now observe that $g(S) =  \sum_{1\le i < j \le n} \frac{2}{n^2} (y_i - y_j)^2$. By deleting one example from $S$, the summation changes in at most $n-1$ terms, whose sum is bounded from above by $\frac{2}{n^2}(n-1) < \frac{2}{n} = \frac{2}{\max(|S|,|S'|)}$.
\end{proof}
\begin{lemma}\label{lem:gini_lips}
Let $S,S' \in (\R^d,\{0,1\})^{\ge 1}$. If $\ED(S,S') \le 1$, then for any $j \in [d]$ and any $t \in \R$,
\begin{align}
\Big|g(S[x_j \le t]) - g(S'[x_j \le t])\Big| \le \frac{12}{\max(|S|,|S'|)}
\end{align}
and the same holds with $>$ in place of $\le$.
\end{lemma}
\begin{proof}
For brevity define $S_- = S[x_j \le t]$, $S_+ = S[x_j > t]$, $S'_- = S'[x_j \le t]$, $S'_+ = S'[x_j > t]$.
By definition of gain, by the triangle inequality, and by rearranging terms, we have:
\begin{align}
\Big|g(S'_-) - g(S_-)\Big| & \le \left| g(S') - g(S) \right| \nonumber \\&+ \left|\frac{|S'_-|}{|S'|} g(S'_-) -  \frac{|S_-|}{|S|} g(S_-) \right|\nonumber \\&+ \left|\frac{|S'_+|}{|S'|} g(S'_+) - \frac{|S_+|}{|S|} g(S_+)  \right|
\label{eq:GRS}
\end{align}
By Lemma~\ref{lem:smooth_gini}, the first term is bounded by $\frac{2}{\max(|S|,|S'|)}$. For the second term, we have:
\begin{align}
    \frac{|S'_-|}{|S'|} g(S'_-)
    & \le \frac{|S'_-|}{|S'|}\! \left(g(S_-) + \frac{2}{\max(|S_-|,|S'_-|)}\right)\label{eq:A1_1}
    \\& \le \frac{|S_-|}{|S'|}g(S_-) + \frac{2}{|S'|} 
    \\& \le \frac{|S'_-|+1}{|S'|+1} g(S_-) + \frac{2}{|S'|} 
    \\& \le \frac{|S_-|+2}{|S|} g(S_-) + \frac{2}{|S'|} \label{eq:A1_2}
    \\& \le \frac{|S_-|}{|S|} g(S_-) + \frac{1}{|S|} + \frac{2}{|S'|} \label{eq:A1_3}
    \\& \le \frac{|S_-|}{|S|} g(S_-) + \frac{5}{\max(|S|,|S'|)} \label{eq:A1_4}
\end{align}
where~\eqref{eq:A1_1} follows from Lemma~\ref{lem:smooth_gini},~\eqref{eq:A1_1} from $ \ED(S'_-,S_-), \ED(R,S)\le 1$,~\eqref{eq:A1_3} from $g(\cdot) \le \frac{1}{2}$, and~\eqref{eq:A1_4} from $\frac{1}{2} \le \frac{|S|}{|S'|} \le 2$.
By exchanging $S$ and $R$,~\eqref{eq:A1_1}--\eqref{eq:A1_4} also implies that $\frac{|S_-|}{|S|} g(S_-) \le \frac{|S'_-|}{|S'|} g(S'_-) + \frac{5}{\max(|S|,|S'|)}$. Together, these two inequalities imply that the second term in the right-hand side of~\eqref{eq:GRS} is bounded by $\frac{5}{\max(|S|,|S'|)}$. A similar argument gives the same bound for the third term. Thus, from~\eqref{eq:GRS}:
\begin{align}
    |g(S,j,t) - g(S',j,t)| \le \frac{2}{\max(|S|,|S'|)} \\+ 2 \cdot \frac{5}{\max(|S|,|S'|)} \\= \frac{12}{\max(|S|,|S'|)}
\end{align}
The argument is symmetric for $>$, concluding the proof.
\end{proof}

\subsubsection{Proof of Lemma~\ref{lem:smoothness}.}
Let us prove the claim on the Gini gain. If $\EDR(S,S') \ge \frac{1}{25}$, then $12.5 \EDR(S,S') \ge \frac{1}{2}$, and the claim holds since $G \in [0,.5]$. Suppose then  $\EDR(S,S') < \frac{1}{25}$. Let $\ED(S,S') = \EDR(S,S') \max(|S|,|S'|)$. Then there is a sequence of labeled sequences $S_0,\ldots,S_{\ell}$  with $\ell = \ED(S,S')$ such that $S=S_0$, $S_{\ell}=S'$, and $\ED(S_{i},S_{i+1}) = 1$ for all $i=0,\ldots,\ell-1$. By Lemma~\ref{lem:gini_lips}, $|g(S_{i-1},j,t) - g(S_i,j,t)| \le \frac{12}{\max(|S_{i-1}|,|S_i|)}$. Moreover, $|S_i| \ge |S| - \ED(S,S')$ and $|S_i| \ge |S'| - \ED(S,S')$, which implies:
\begin{align}
    |S_i| &\ge \max(|S|,|S'|)-\ED(S,S') \\&\ge (1-\EDR(S,S')) \max(|S|,|S'|) \\&> \frac{24}{25} \max(|S|,|S'|)
\end{align}
Then:
\begin{align}
    |g(S,j,t)\!-\!g(S'\!,j,t)| &\le \sum_{i=0}^{\ell-1}\big|g(S_i,j,t)\!-\!g(S_{i+1},j,t)\big| \label{eq:L2.2_1}
    \\&\le \sum_{i=0}^{\ell-1} \frac{12}{\max(|S_i|,|S_{i+1}|)} \label{eq:L2.2_2} 
    \\&\le \sum_{i=0}^{\ell-1} \frac{12}{\frac{24}{25}\max(|S|,|S'|)} \label{eq:L2.2_3}
    \\&\le \ED(S,S') \frac{12.5}{\max(|S|,|S'|)} \label{eq:L2.2_4}
    \\&= 12.5\, \EDR(S,S')
\end{align}
where~\eqref{eq:L2.2_1} follows from the triangle inequality,~\eqref{eq:L2.2_2} from Lemma~\ref{lem:gini_lips},~\eqref{eq:L2.2_3} from above, and ~\eqref{eq:L2.2_4} from $\ell = \ED(S,S')$. This concludes the proof.

For the Gini index the proof is identical, except for the constants and for Lemma~\ref{lem:gini_lips} in place of Lemma~\ref{lem:smooth_gini}. If $\EDR(S,S') \ge \frac{1}{5}$ then $2.5 \EDR(S,S') = \frac{1}{2}$ and we are done as $g \in [0,0.5]$. Otherwise consider the sequence $S_0,\ldots,S_{\ell}$ and, like above, note that $|S_i| > \frac{4}{5} \max(|S|,|S'|)$. A chain of inequalities similar to the one above yields the result.

\subsection{Supplementary material for Section~\ref{sec:dyn}}
\label{apx:dyn}
\begin{lemma}\label{lem:update_correct}
The tree computed by \AlgoUpdate\ over any sequence of update requests $U=o_1(\bx_1,y_1),\ldots,o_t(\bx_t,y_t)$ is $\beps$-feasible w.r.t.\ the active sequence $S^t$.
\end{lemma}
\begin{proof}
Let $T$ be the decision tree currently held by \AlgoUpdate, let $S$ be the current active sequence. Suppose by contradiction that $T$ is not $\beps$-feasible w.r.t.\ $S$, so some $v \in V(T)$ violates one of the conditions of Definition~\ref{def:eps_feasibility}. Let $S_v^0$ be the sequence used by the most recent (and possibly nested) invocation of \AlgoBuild\ that constructed the subtree $T_v$. Note that the values of $c(v)$ and $s(v)$ currently held by \AlgoUpdate\ satisfy:
\begin{align}
    \frac{c(v)}{s(v)} \ge \frac{\ED(S_v^0,S_v)}{|S_v^0|} \ge \frac{\ED(S_v^0,S_v)}{\max(|S_v^0|,|S_v|)} = \EDR(S_v^0,S_v) \label{eq:cv_sv_lb}
\end{align}
We now consider the conditions of Definition~\ref{def:eps_feasibility} separately. We show that if $v$ violates any of them then $\EDR(S_v^0,S_v) > \epsilon$. By the inequality above this implies $c(v) > s(v) \cdot \eps$, which is absurd as \AlgoUpdate\ ensures $c(v) \le s(v) \cdot \eps$ at every instant. First, observe that condition (3) of Definition~\ref{def:eps_feasibility} is trivially satisfied by \AlgoUpdate, hence we need only consider conditions (1) and (2).

\textbf{Case 1:} $v$ violates condition (1). Then one of the following two subcases holds.\\
\textbf{Case 1a:} $v$ is a leaf such that $|S_v|>k$ and $g(S_v) \ge \epsThr$ and $\depth_T(v) < h$. Clearly the depth of $v$ was $\depth_T(v)$ also just before $T_v$ was constructed; therefore at that time $|S_v^0|\le k$ or $g(S_v^0) \le \frac{\epsThr}{2}$, otherwise $v$ would not have been a leaf according to \AlgoBuild. Now, if $|S_v^0|\le k$, then $\EDR(S_v,S_v^0) \ge \frac{1}{k}$ as $|S_v| > k$. If instead $g(S_v^0) \le \frac{\alpha}{2}$, then by Lemma~\ref{lem:smoothness} and since $g(S_v) \ge \epsThr$ we have $\EDR(S_v,S_v^0) \ge \frac{\epsThr}{5}$.
Thus $\EDR(S_v,S_v^0) \ge \min\left(\frac{1}{k},\frac{\epsThr}{5}\right)$.

\noindent \textbf{Case 1b:} $v$ is an internal vertex such that $|S_v|\le k$ or $g(S_v) = 0$ or $\depth_T(v)=h$. Clearly the depth of $v$ was $\depth_T(v)$ also just before $T_v$ was constructed; therefore at that time $|S_v^0| \ge k+1$ or $g(S_v^0) > \frac{\epsThr}{2}$, otherwise $v$ would have not been internal according to \AlgoBuild. The same arguments of case 1a show that $\EDR(S_v,S_v^0) \ge \min\big(\frac{1}{k+1},\frac{\epsThr}{5}\big)$.

\textbf{Case 2:} $v$ violates condition (2). Hence $G(S_v,j,a) < G(S_v,j^*,a^*) - \epsApx$, where $(j,a)$ is the split used by $v$ and $(j^*,a^*)=\arg\max_{j,a}\{G(S_v,j,a)\}$ is the optimal split. Since $(j,a)$ was chosen as a split, then it had maximal Gini gain for $S_v^0$, therefore $G(S_v^0,j,a) \ge G(S_v^0,j^*,a^*)$. Thus:
\begin{align}
    G(S_v,j^*,a^*) - G(S_v,j,a) &> \epsApx \\
    G(S_v^0,j^*,a^*) - G(S_v^0,j,a) &\le 0
\end{align}
and therefore
\begin{align}
    \left(G(S_v,j^*,a^*)-G(S_v^0,j^*,a^*)\right) \nonumber\\+ \left(G(S_v^0,j,a)-G(S_v,j,a)\right) &> \epsApx 
\end{align}
which implies
\begin{align}
    \max\left(\left|G(S_v,j^*,a^*)-G(S_v^0,j^*,a^*)\right|\right.,\nonumber\\ \left.\left|G(S_v^0,j,a)-G(S_v,j,a)\right| \right)> \epsApx 
\end{align}
By Lemma~\ref{lem:smoothness} this implies $\EDR(S_v^0,S_v) > \frac{\epsApx}{12.5}$.

Combining the bounds on $\EDR(S_v,S_v^0)$ and using the definition of $\epsilon$, we obtain:
\begin{align}
    \EDR(S_v,S_v^0) \ge \min\left(\frac{1}{k+1},\frac{\epsThr}{5},\frac{\epsApx}{12.5}\right) > \epsilon
\end{align}
which as noted above yields a contradiction.
\end{proof}

\subsubsection{Proof of Lemma~\ref{lem:counter}.}
Suppose by contradiction that $|\,|S^t(v)|-s^{t}(v)| > \epsilon \cdot s^{t}(v)$ and let $t_0 \le t$ be the last time a subtree containing $v$ was rebuilt.
Note that $c^t(v) \ge |\,|S^t_v|-s^t(v)|$ by construction of \AlgoUpdate, hence $c^t(v) > \epsilon \cdot s^{t}(v) > 0$. This implies $t_0 < t$, and since $c^{t_0}(v)=0 < \epsilon \cdot s^{t_0}(v)$, it also implies the existence of a smallest $\tau \in [t_0+1,t]$ such that $c^{\tau}(v) > \epsilon \cdot s^{\tau}(v)$. But then, by construction of \AlgoUpdate, at time $\tau$ \AlgoBuild\ was executed,  yielding a contradiction.

\subsubsection{Proof of Lemma~\ref{lem:log_height}.}
Fix any $v \in V(T)$ and let $(j,t)$ be its split. Let $u,z$ be the children of $v$, with $|S_u| \le |S_z|$, and let $\nu = |S_u|/|S_v|$. Using Definition~\ref{def:gini} and straightforward manipulations:
\begin{align}
    \gamma \le G(S_v,j,t) &\le g(S_v) - (0 + (1-\nu)g(S_z))
    \\ &\le g(S_v) - g(S_z) + \nu
\end{align}
As $|S_z| = (1-\nu) |S_v|$ then $\EDR(S_z, S_v) = \nu$ and Lemma~\ref{lem:smoothness} yields $g(S_v) - g(S_z) \le 2.5 \nu$. We conclude that $\gamma < 4 \nu$, i.e., $\nu > \gamma/4$. Therefore $\min(|S_u|,|S_z|) > \frac{\gamma}{4} |S_v|$. This holds for all $v$, thus $T$ has height $\scO(\log_{1+\gamma}|S|) = \scO \big( \frac{\log |S|}{\gamma} \big) $.

\subsubsection{Proof of Theorem~\ref{thm:categorical_ub}.}\label{sec:categ}
We show that, when all features are categorical, \AlgoBuild\ can be implemented in time $O(n d \log n)$ where $n=|S|$. The claim then follows by an argument similar to the one in the proof of Lemma~\ref{lem:amtime}.
The main idea is to maintain the following statistics and data structures: 
\begin{itemize}
    \item for every $y \in \{0,1\}$, the number $n(S,y)$ of examples in $S$ with label $y$
    \item for every $j \in [d]$, every $a \in A$, and every $y \in \{0,1\}$, the number $n(S,j,a,y)$ of examples in $S$ with label $y$ whose $j$-th feature has value $a$
    \item for every $j \in [d]$, every $i \in [n]$, a bidirectional pointer to the next example $\bz_l$ in the sequence such that $\bz_{lj}=\bx_{ij}$.     
\end{itemize}

Note that the above data structure and counters can be computed altogether in time $O(nd)$ by traversing $S$. The bidirectional pointers allow us to enumerate all examples with value $a$ in features $j$ in time $O(|\bar{S}|)$, where $\bar{S}$ denotes such a set of examples. Observe that, for every $j \in [d]$, in time $O(|A|)=O(1)$ we can compute $G(S,j,t)$ from $n(S,j,\cdot,\cdot)$. Hence, in time $O(d)$ we can find  $(j^*,t^*)=\arg\min\{G(S,j,t):j \in [d], t \in A\}$.

When making a split in the tree for a given node $v$, we first determine the feature $j$ with maximum Gini gain in $S(v)$. Then, we compute $a = \arg \min_{a \in A} \left( n(S(v),j,a,0) + n(S(v),j,a,1) \right) $. After that, we create a new sequence $\bar{S}(v)$ containing all examples with value $a$ in feature $j$, which are removed from $S(v)$, i.e. $S(v) \gets S(v) \setminus \bar{S}(v)$. We update all counters and pointers for the $S(v)$ and $\bar{S}(v)$, which can be done in time $O(d |\bar{S}(v)| )$. From the fact, that for every node $v$, $ |\bar{S}(v)| \leq \frac12 \cdot |S(v)|$, it follows that the total running time of \AlgoBuild\ when all features are categorical is $O(n d \log n)$.

\subsection{Supplementary material for Section~\ref{sec:lb}}
\subsubsection{Proof of Theorem~\ref{thm:time_lb}.}
We prove a more general claim, that is, that for every $\delta > 0$ any $(\beps,\nicefrac{1}{2}+\delta)$-feasible fully dynamic algorithm has expected running time $\Omega(\delta nd)$.

Let $d \ge 2$. For $j \in [d]$ let $M^{(j)}$ be the ${2k \times (d+1)}$ boolean matrix defined as follows. The first $k$ rows are set to zero. The last $k$ rows have the $j$-th entry and the $(d+1)$-th entry set to one. Let $S^{(j)}$ be the sequence having the rows of $M^{(j)}$ as elements, consider any decision tree $(\beps,0)$-feasible w.r.t.\ $S^{(j)}$, and let $u$ be its root. Since $|S^{(j)}|>k$ and $g(S)=\frac{1}{2}>\epsThr$ and $u$ has depth $0<h$, then $u$ is not a leaf. Moreover, since $g(S^{(j)},j)=\frac{1}{2}$ while $g(S^{(j)},j')=0 \; \forall j' \in [d] \setminus \{j\}$, $u$ splits on $j$. It is immediate to see that the children of $u$ are leaves with labelling rule $\Ind_{\{x_j=1\}}$, so this is the only feasible class for any $x$.

Now let $\scA$ be a fully dynamic algorithm and let $j \in [d]$; note that $j$ is unknown to $\scA$. We feed $\scA$ with $\ins(x)$ for every row $x$ of $M^{(j)}$. Then, we ask $\scA$ to label the random vector $x \in \{0,1\}^d$ chosen uniformly at random among all vectors having exactly $\lceil \nicefrac{d}{2} \rceil$ features set to $1$. As shown above, $\scA(x)=\Ind_{\{x_j=1\}}$. Since $j$ is unknown to $\scA$, deciding if $x_j=1$ with probability $\nicefrac{1}{2}+\delta$ requires looking at $\Omega(\delta d)$ bits of the active sequence in expectation. Hence, $\scA$ has expected running time $\Omega(\delta d)$.

Iterating the construction yields an asymptotic time bound. To this end simply ask $\scA$ to delete the $2k$ examples inserted, and repeat. This yields an expected running time of $\Omega(nd)$ where $n$ is the total number of insertions and deletions. 

\subsection{Additional Experiments}\label{sec:addexp}

In this section, we report additional experiments, in particular for those datasets that have not been considered in the main part of our paper, while we also include a comparison with HAT. Figures~\ref{fig:sw_add}, \ref{fig:RU_f1}, and~\ref{fig:ru_add} confirm the trends observed in the main part of our paper. In particular, the running time decreases by several orders of magnitude as $\epsilon$ increases, while the F1-score is relatively less affected. Figure~\ref{fig:sw_add} reports the F1-score and average running time in the SW model, for the remaining three datasets. Figure~\ref{fig:RU_f1} shows the F1-score of \algo\ in the RU model for the three datasets considered in the main part of the paper, while Figure~\ref{fig:ru_add} shows F1-score and average running time in the RU model, for the remaining three datasets.
\begin{figure*}
    \centering
    \begin{subfigure}[b]{0.3\textwidth}
         \centering
         \includegraphics[width=\textwidth]{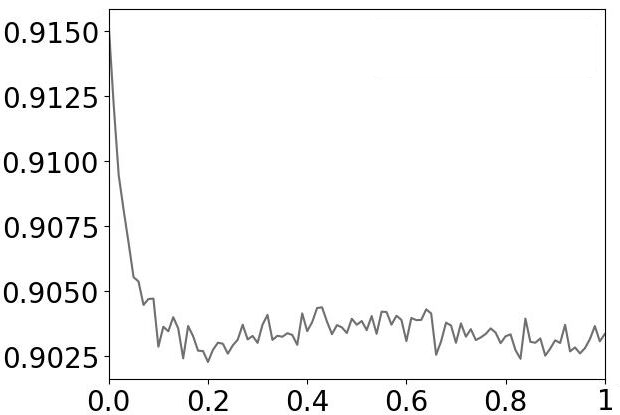}
         \label{fig:f1_cov}    
    \end{subfigure}
    \hfill
    \begin{subfigure}[b]{0.3\textwidth}
         \centering
         \includegraphics[width=\textwidth]{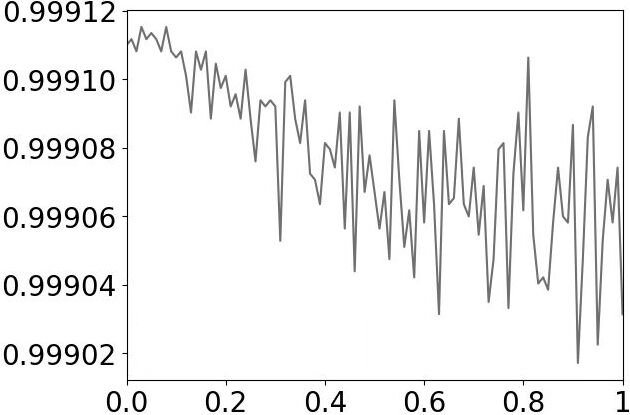}
         \label{fig:f1_kdd}    
    \end{subfigure}
    \hfill
    \begin{subfigure}[b]{0.3\textwidth}
         \centering
         \includegraphics[width=\textwidth]{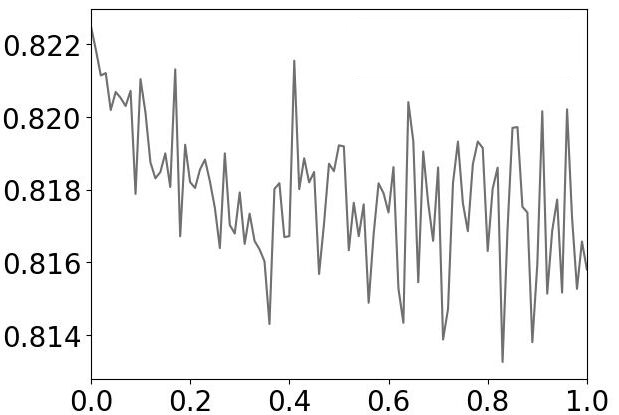}
         \label{fig:f1_noaa}    
    \end{subfigure}
    \\
    \label{fig:f1_scores_add}
    \begin{subfigure}[b]{0.3\textwidth}
         \centering
         \includegraphics[width=\textwidth]{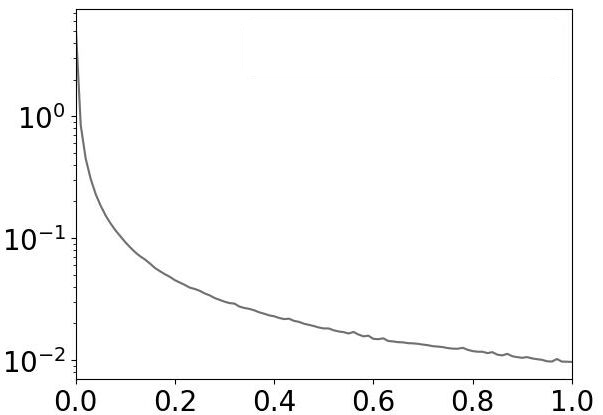}
         \label{fig:time_cov}    
    \end{subfigure}
    \hfill
    \begin{subfigure}[b]{0.3\textwidth}
         \centering
         \includegraphics[width=\textwidth]{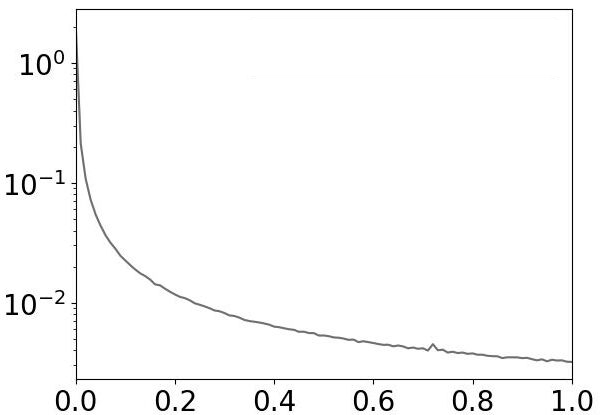}
         \label{fig:time_kdd}    
    \end{subfigure}
    \hfill
    \begin{subfigure}[b]{0.3\textwidth}
         \centering
         \includegraphics[width=\textwidth]{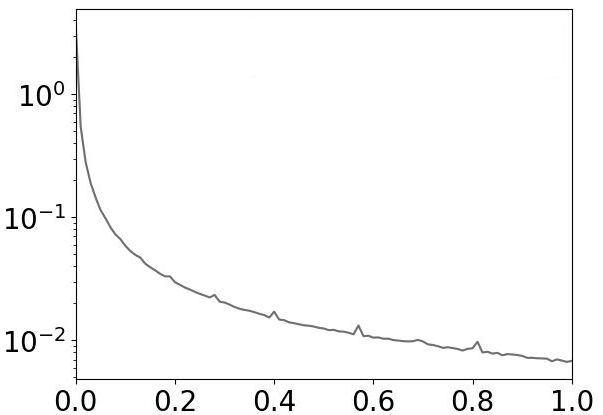}
         \label{fig:time_noaa}    
    \end{subfigure}
\caption{F1-score (top) and amortized running time (bottom) in the SW model on the Forest Covertype, KDDCUP99 and NOAA Weather datasets (left to right).}\label{fig:sw_add}
\end{figure*}
\begin{figure*}[h]
    \centering
    \begin{subfigure}[b]{0.3\textwidth}
         \centering
         \includegraphics[width=\textwidth]{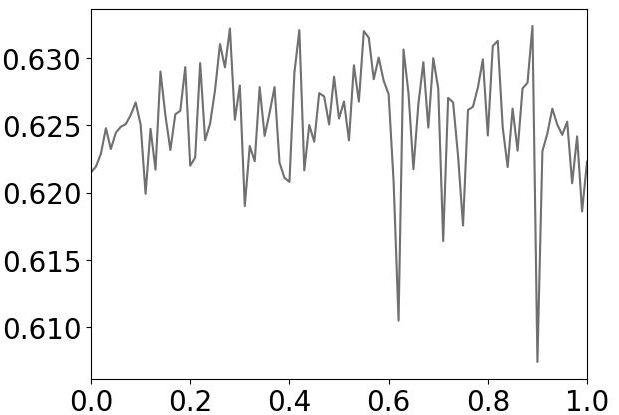}
         \label{fig:ru_elec_f1}    
    \end{subfigure}
    \hfill
    \begin{subfigure}[b]{0.28\textwidth}
         \centering
         \includegraphics[width=\textwidth]{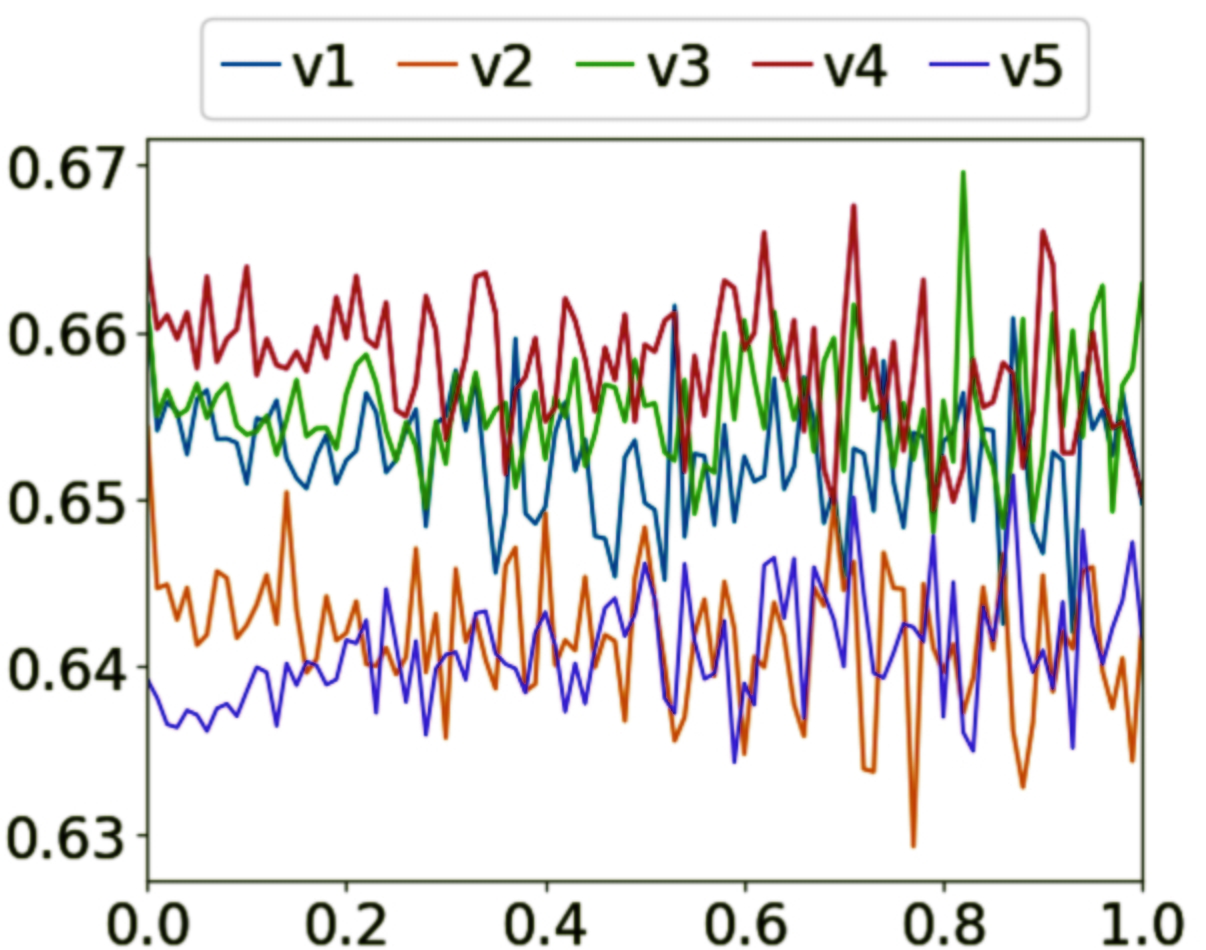}
         \label{fig:ru_insects_f1}    
    \end{subfigure}
    \hfill
    \begin{subfigure}[b]{0.3\textwidth}
         \centering
         \includegraphics[width=\textwidth]{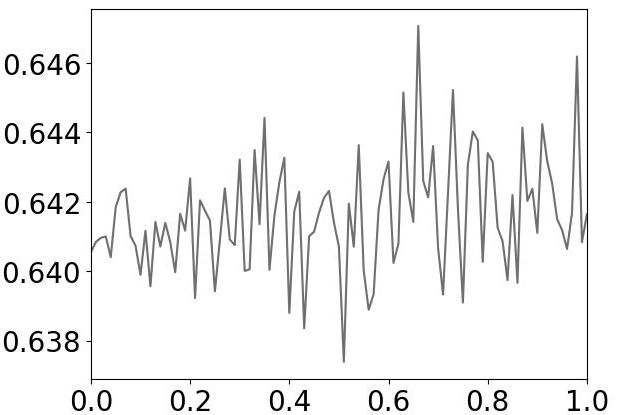}
         \label{fig:ru_poker_f1}    
    \end{subfigure}
    \caption{F1-scores in the RU model on the Electricity, INSECTS and Poker datasets.}\label{fig:RU_f1}
\end{figure*}
\begin{figure*}
    \centering
    \begin{subfigure}[b]{0.3\textwidth}
         \centering
         \includegraphics[width=\textwidth]{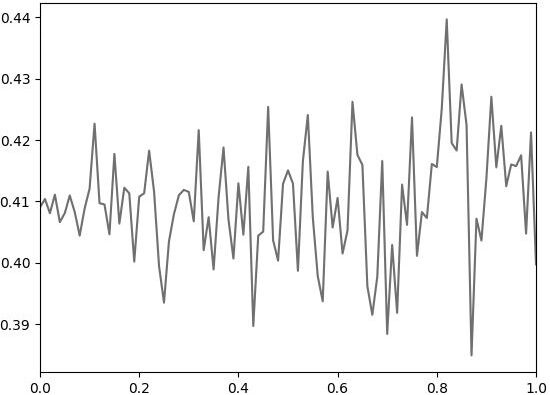}
         \label{fig:f1_cov_ru}    
    \end{subfigure}
    \hfill
    \begin{subfigure}[b]{0.3\textwidth}
         \centering
         \includegraphics[width=\textwidth]{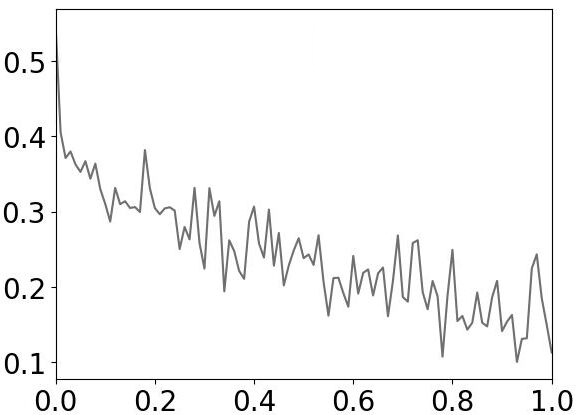}
         \label{fig:f1_kdd_ru}    
    \end{subfigure}
    \hfill
    \begin{subfigure}[b]{0.3\textwidth}
         \centering
         \includegraphics[width=\textwidth]{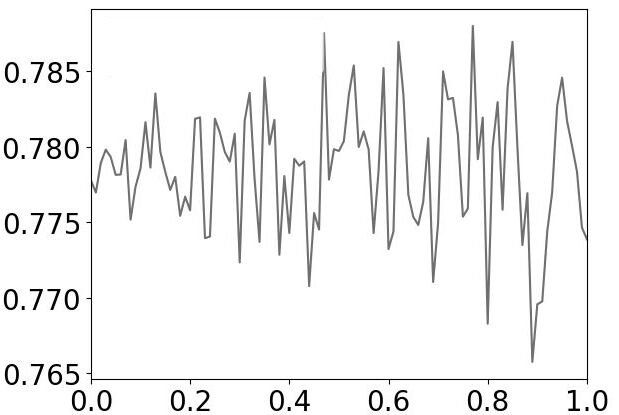}
         \label{fig:f1_noaa_ru}    
    \end{subfigure}
    \label{fig:ru_f1_scores_add}
    \begin{subfigure}[b]{0.3\textwidth}
         \centering
         \includegraphics[width=\textwidth]{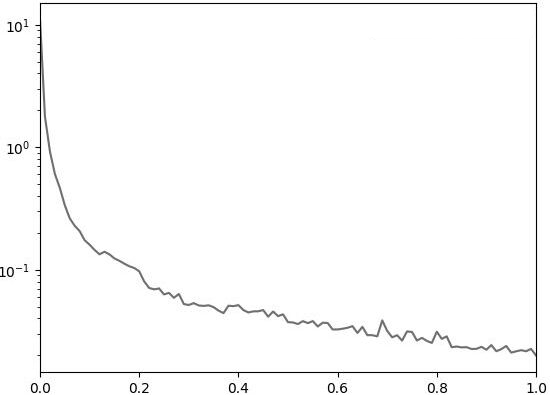}
         \label{fig:time_cov_ru}    
    \end{subfigure}
    \hfill
    \begin{subfigure}[b]{0.3\textwidth}
         \centering
         \includegraphics[width=\textwidth]{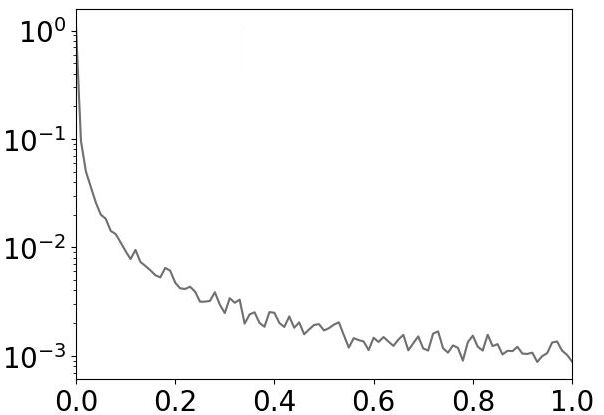}
         \label{fig:time_kdd_ru}    
    \end{subfigure}
    \hfill
    \begin{subfigure}[b]{0.3\textwidth}
         \centering
         \includegraphics[width=\textwidth]{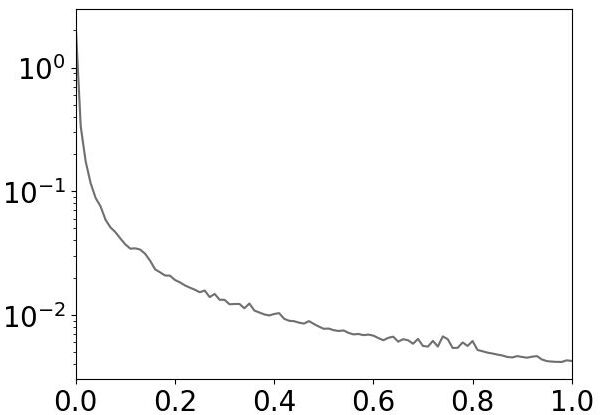}
         \label{fig:time_noaa_ru}    
    \end{subfigure}
\caption{F1-score (top) and amortized running time (bottom) in the RU model on the Forest Covertype, KDDCUPP99 and NOAA Weather datasets (left to right).}\label{fig:ru_add}
\end{figure*}
Table~\ref{tab:grace_periods} shows the optimal grace periods for HAT and EFDT. Table~\ref{tab:exec_time_compare}
reports the total running times of the three different approaches, while Table~\ref{tab:f1_compare} reports the corresponding F1-scores. In all those experiments, we let $\epsilon=0.5$ for our algorithm \algo. Those results confirm our findings, namely, that our algorithm outperforms the other approaches in terms of F1-score, even in the incremental setting. 

\begin{table*}[h]
    \centering
    \begin{tabular}{lll}
        \toprule
        Dataset & Grace Period (HAT) & Grace Period (EFDT) \\
        \midrule
        Electricity & 1000 & 500 \\
        Forest Covertype & 1000 & 100 \\
        INSECTS v1 & 1000 & 500 \\
        INSECTS v2 & 500 & 100 \\
        INSECTS v3 & 500 & 100 \\
        INSECTS v4 & 1000 & 100 \\
        INSECTS v5 & 500 & 500 \\
        KDDCUP99 & 1000 & 500 \\
        NOAA Weather & 1000 & 500 \\
        Poker & 500 & 100 \\
        \bottomrule
    \end{tabular}
    \caption{Grace Periods of best F1-scores (incremental setting). }
    \label{tab:grace_periods}
\end{table*}

\begin{table*}[h]
    \centering
    \begin{tabular}{llll}
        \toprule
        Dataset & HAT & EFDT & \algo\\
        \midrule
        Electricity & 2.06 & 1.65 & \textbf{0.76} \\
        Forest Covertype & 39.95 & 42.47 & \textbf{30.78}\\
        INSECTS v1 & \textbf{4.77} & 4.85 & 11.28 \\
        INSECTS v2 & \textbf{2.53} & 3.13 & 8.34\\
        INSECTS v3 & \textbf{6.79} & 7.54 & 15.86 \\
        INSECTS v4 & \textbf{5.73} & 6.30 & 12.74 \\
        INSECTS v5 & 7.31 & \textbf{6.84} & 16.85 \\
        KDDCUP99 & 18.28 & 17.72 & \textbf{10.49} \\
        NOAA Weather & 0.82 & 0.73 & \textbf{0.66} \\
        Poker & \textbf{16.78} & \textbf{16.26} & 36.41 \\
        \bottomrule
    \end{tabular}
    \caption{Total execution times in seconds (incremental setting).}
    \label{tab:exec_time_compare}
\end{table*}

\begin{table*}[h]
    \centering
    \begin{tabular}{llll}
        \toprule
        Dataset & F1-score (HAT) & F1-score (EFDT) & F1-score (\algo)\\
        \midrule
        Electricity & 70.81 & 72.05 & \textbf{82.12} \\
        Forest Covertype & 84.30 & 83.64 & \textbf{90.37} \\
        INSECTS v1 & 87.35 & 88.96 & \textbf{92.90} \\
        INSECTS v2 & 87.61 & 87.40 & \textbf{92.78} \\
        INSECTS v3 & 92.32 & 92.51 & \textbf{94.84} \\
        INSECTS v4 & 91.32 & 91.15 & \textbf{91.85} \\
        INSECTS v5 & 89.95 & 89.85 & \textbf{93.48} \\
        KDDCUP99 & \textbf{99.97} & 97.98 & 99.91 \\
        NOAA Weather & 80.58 & 80.78 & \textbf{81.92} \\
        Poker & 67.93 & 79.69 & \textbf{88.34} \\
        \bottomrule
    \end{tabular}
    \caption{F1-scores (incremental setting).}
    \label{tab:f1_compare}
\end{table*}

\end{document}